\documentclass[10pt]{article}

\usepackage[utf8]{inputenc} 
\usepackage[T1]{fontenc}    
\usepackage{hyperref}       
\usepackage{url}            
\usepackage{booktabs}       
\usepackage{amsfonts}       
\usepackage{nicefrac}       
\usepackage{microtype}      
\usepackage{xcolor}         
\usepackage{fullpage}

\usepackage{amsmath}
\usepackage{amssymb}
\usepackage{mathtools}
\usepackage{amsthm}
\usepackage{algorithm}
\usepackage{algorithmic}

\usepackage{cleveref}
\usepackage{natbib}

\usepackage{colortbl}
\usepackage{multirow}

  \newtheorem{theorem}{Theorem}
  \newtheorem{lemma}{Lemma}

  \newtheorem{corollary}{Corollary}
  \newtheorem{definition}{Definition}
  
  \newtheorem{assum}{Assumption}

\usepackage{graphicx}

\definecolor{light-gray}{gray}{0.89}

\def\eqref#1{equation~\ref{#1}}

\def\1{\bm{1}}

\DeclareMathAlphabet{\mathsfit}{\encodingdefault}{\sfdefault}{m}{sl}
\SetMathAlphabet{\mathsfit}{bold}{\encodingdefault}{\sfdefault}{bx}{n}

\DeclareMathOperator*{\argmin}{arg\,min}

\title{Will Bilevel Optimizers Benefit from Loops}

\author
{
	Kaiyi Ji\thanks{Department of EECS, University of Michigan, Ann Arbor; e-mail: {\tt  kaiyiji@umich.edu}}
	,~~~Mingrui Liu\thanks{Department of CS, George Mason University; e-mail: {\tt   mingruil@gmu.edu}} 
	,~~~Yingbin Liang\thanks{Department of ECE, The Ohio State University; e-mail: {\tt liang.889@osu.edu}}
	~~~and~~~Lei Ying\thanks{Department of EECS, University of Michigan, Ann Arbor; e-mail: {\tt   leiying@umich.edu}}
}

\begin{document}

\maketitle

\begin{abstract}%
Bilevel optimization has arisen as a powerful tool for solving a variety of machine learning problems. Two current popular bilevel optimizers AID-BiO and ITD-BiO naturally involve solving one or two sub-problems, and consequently, whether we solve these problems with loops (that take many iterations) or without loops (that take only a few iterations) can significantly affect the overall computational efficiency. Existing studies in the literature cover only some of those implementation choices, and the complexity bounds available are not refined enough to enable rigorous comparison among different implementations. In this paper, we first establish unified convergence analysis for both AID-BiO and ITD-BiO that are applicable to all implementation choices of loops. We then specialize our results to characterize the computational complexity for all implementations, which enable an explicit comparison among them. Our result indicates that for AID-BiO, the loop for estimating the optimal point of the inner function is beneficial for overall efficiency, although it causes higher complexity for each update step, and the loop for approximating the outer-level Hessian-inverse-vector product reduces the gradient complexity. For ITD-BiO, the two loops always coexist, and our convergence upper and lower bounds show that such loops are necessary to guarantee a vanishing convergence error, whereas the no-loop scheme suffers from an unavoidable non-vanishing convergence error. Our numerical experiments further corroborate our theoretical results.
\end{abstract}

\section{Introduction} 
Bilevel optimization has attracted significant attention recently due to its popularity in a variety of machine learning applications including meta-learning~\citep{franceschi2018bilevel,bertinetto2018meta,rajeswaran2019meta,ji2020convergence}, hyperparameter optimization~\citep{franceschi2018bilevel,shaban2019truncated,feurer2019hyperparameter}, reinforcement learning~\citep{konda2000actor,hong2020two}, and signal processing~\citep{kunapuli2008classification,flamary2014learning}. In this paper, we consider the bilevel optimization problem that takes the following formulation. 
\begin{align}\label{objective_deter}
\min_{x\in\mathbb{R}^{p}} \Phi(x):=f(x, y^*(x)) \quad\mbox{s.t.} \quad y^*(x)= \argmin_{y\in\mathbb{R}^{q}} g(x,y),
\end{align}
where the outer- and inner-level functions $f$ and $g$ are both jointly continuously differentiable. 
We focus on the setting where the lower-level function $g$ is strongly convex with respect to (w.r.t.)~$y$ with the condition number $\kappa = \frac{L}{\mu}$ (where $L$ and $\mu$ are gradient Lipschitzness and strong convexity coefficients defined respectively in Assumptions \ref{assum:geo} and \ref{high_lip} in \Cref{sec:def}), and the outer-level objective function $\Phi(x)$ is possibly nonconvex w.r.t.~$x$. Such types of geometries arise in many applications including meta-learning (which uses the last layer of neural networks as adaptation parameters), hyperparameter optimization (e.g., data hyper-cleaning and regularized logistic regression) and learning in communication networks (e.g., network utility maximization).  

A variety of algorithms have been proposed to solve the bilevel optimization problem in \cref{objective_deter}. For example, ~\cite{hansen1992new,shi2005extended,moore2010bilevel} proposed constraint-based approaches by replacing the inner-level problem with its optimality conditions as constraints. In comparison, gradient-based bilevel algorithms have received intensive attention recently due to the effectiveness and simplicity, which include two popular approaches via approximate implicit differentiation (AID)~\citep{domke2012generic,pedregosa2016hyperparameter,grazzi2020iteration,ji2021bilevel} and iterative differentiation (ITD)~\citep{maclaurin2015gradient,franceschi2017forward,shaban2019truncated}. Readers can refer to \Cref{app:related_work} for an expanded list of related work.   

\begin{table*}[!t]
\renewcommand{\arraystretch}{1.5}
\centering
\caption{Comparison of computational complexities of four AID-BiO implementations for finding an $\epsilon$-accurate stationary point. For a fair comparison, gradient descent (GD) is used to solve the linear system for all algorithms. 
MV$(\epsilon)$: the total number of Jacobian- and Hessian-vector product computations. 
Gc$(\epsilon)$: the total number of gradient computations. 
$\mathcal{\widetilde O}$: hide $\ln\frac{\kappa}{\epsilon}$ factors. }
\label{tab:results}
\vspace{0.3cm}
\begin{tabular}{|c|c|c|c|c|} \hline
 \textbf{Algorithms} & $Q$ & $N$ &\textbf{MV}($\epsilon$) & \textbf{Gc}($\epsilon$) 
 \\ \hline 
BA~\citep{ghadimi2018approximation} & $\Theta(\kappa\ln\kappa)$   & $\frac{(k+1)^{\frac{1}{4}}}{2}$ ($k$: iteration number)   & {\small $\mathcal{\widetilde O}(\kappa^{5}\epsilon^{-1})$}  & {\small $\mathcal{\widetilde O}(\kappa^{5}\epsilon^{-1.25})$} 
\\ \cline{1-5}
AID-BiO~\citep{ji2021bilevel}& $\Theta(\kappa\ln\kappa)$ & $\Theta(\kappa\ln\kappa)$  &  {\small $\mathcal{\widetilde O}(\kappa^{4}\epsilon^{-1})$}  &  {\small $\mathcal{\widetilde O}(\kappa^{4}\epsilon^{-1})$} 
\\ \cline{1-5}
$N$-$Q$-loop AID (this paper)  \cellcolor{blue!15} &\cellcolor{blue!15} $\Theta(\kappa\ln\kappa)$  &  \cellcolor{blue!15} $\Theta(\kappa\ln\kappa)$ & \cellcolor{blue!15} {\small $\mathcal{\widetilde O}(\kappa^{4}\epsilon^{-1})$} 
&\cellcolor{blue!15} {\small $\mathcal{\widetilde O}(\kappa^{4}\epsilon^{-1})$} 
\\ \cline{1-5} 
$Q$-loop AID (this paper)\cellcolor{blue!15} &  \cellcolor{blue!15} $\Theta(\kappa\ln\kappa)$ &\cellcolor{blue!15} $1$ & \cellcolor{blue!15}
{\small $\mathcal{\widetilde O}(\kappa^{6}\epsilon^{-1})$} &  \cellcolor{blue!15}
{\small $\mathcal{\widetilde O}(\kappa^{5}\epsilon^{-1})$}  \\ \cline{1-5}
$N$-loop AID (this paper)\cellcolor{blue!15}&  \cellcolor{blue!15} $\mathcal{O}(1)$&  \cellcolor{blue!15} $\Theta(\kappa\ln\kappa)$& \cellcolor{blue!15}
{\small $\mathcal{\widetilde O}(\kappa^{4}\epsilon^{-1})$}
&  \cellcolor{blue!15}
{\small $\mathcal{\widetilde O}(\kappa^{5}\epsilon^{-1})$}
  \\\cline{1-5}
No-loop AID (this paper) \cellcolor{blue!15}&  \cellcolor{blue!15} $\mathcal{O}(1)$&  \cellcolor{blue!15} 1& \cellcolor{blue!15}
{\small $\mathcal{\widetilde O}(\kappa^{6}\epsilon^{-1})$}  & \cellcolor{blue!15}
{\small $\mathcal{\widetilde O}(\kappa^{6}\epsilon^{-1})$} 
\\ \hline
\end{tabular}
\end{table*}


Consider the AID-based bilevel approach (which we call AID-BiO). Its base iteration loop updates the variable $x$ until convergence. Within such a base loop, it needs to solve two sub-problems: finding a nearly optimal solution of the inner-level function via $N$ iterations, and approximating the outer-level Hessian-inverse-vector product via $Q$ iterations. If $Q$ and $N$ are chosen to be large, then the corresponding iterations form {\bf additional loops} of iterations within the base loop, which we respectively call as $Q$-loop and $N$-loop. Thus, AID-BiO can have four popular implementations depending on different choices of $N$ and $Q$: $N$-loop (with large $N=\kappa\ln\kappa$ and small $Q=\mathcal{O}(1)$), $N$-$Q$-loop (with large $N=\Theta(\kappa\ln\kappa)$ and large $Q=\Theta(\kappa\ln\kappa)$), $Q$-loop (with $N=1$ and $Q=\Theta(\kappa\ln\kappa)$), and No-loop (with $N=1$ and $Q=\mathcal{O}(1)$). Note that No-loop refers to no additional loops within the base loop, and can be understood as conventional single-(base)-loop algorithms. These implementations can significantly affect the efficiency of AID-BiO. Generally, large $Q$ (i.e., a $Q$-loop) provides a good approximation of the Hessian-inverse-vector product for the hypergradient computation, and large $N$ (i.e., a $N$-loop) finds an accurate optimal point of the inner function. Hence, an algorithm with $N$-loop and $Q$-loop require fewer base-loop steps to converge, but each such base-loop step requires more computations due to these loops. On the other hand, small $Q$ and/or $N$ avoid computations of loops in each base-loop step, but can cause the algorithm to converge with many more base-loop steps. An intriguing question here is which implementation is overall most efficient and whether AID-BiO benefits from having $N$-loop and/or Q-loop. Existing theoretical studies on AID-BiO are far from answering this question. The studies~\citep{ghadimi2018approximation,ji2021bilevel} on deterministic AID-BiO focused only on the $N$-$Q$-loop scheme. A few studies analyzed the stochastic AID-BiO, such as~\cite{li2021fully} on No-loop, and \cite{hong2020two,khanduri2021near} on $Q$-loop. Those studies were not refined enough to capture the computational differences among different implementations, and further those studies collectively did not cover all the four implementations either.

%

\begin{table*}[!t]
\renewcommand{\arraystretch}{1.5}
\centering
\caption{Comparison of computational complexities of two ITD-BiO implementations for finding an $\epsilon$-accurate stationary point. For a fair comparison, gradient descent (GD) is used to solve the inner-level problem. The analysis in \cite{ji2021bilevel} for ITD-BiO assumes that the inner-loop minimizer $y^*(x_k)$ is bounded at $k^{th}$ iteration, which is not required in our analysis. $\mu$: the strong-convexity constant of inner-level function $g(x,\cdot)$. 
For the last two columns, 'N/A' means that the complexities to achieve an $\epsilon$-accuracy are not measurable due to the nonvanishing convergence error.  
}
\label{tab:results_itd}
\vspace{0.3cm}
\begin{tabular}{|c|c|c|c|c|} \hline
 \textbf{Algorithms} & $N$ &Convergence rate &\textbf{MV}($\epsilon$) & \textbf{Gc}($\epsilon$) 
 \\ \hline 
ITD-BiO~\citep{ji2021bilevel}    & $\Theta(\kappa\ln\kappa)$ & $\mathcal{O}\Big( \frac{\kappa^3}{K} +\epsilon\Big)$ & {\small $\mathcal{\widetilde O}(\kappa^{4}\epsilon^{-1})$} &  {\small $\mathcal{\widetilde O}(\kappa^{4}\epsilon^{-1})$}
\\ \cline{1-5}
$N$-$N$-loop ITD (this paper) \cellcolor{blue!15}   &  \cellcolor{blue!15} $\Theta(\kappa\ln\kappa)$ &\cellcolor{blue!15} $\mathcal{O}\Big( \frac{\kappa^3}{K} +\epsilon\Big)$& \cellcolor{blue!15} {\small $\mathcal{\widetilde O}(\kappa^{4}\epsilon^{-1})$}  &  \cellcolor{blue!15} {\small $\mathcal{\widetilde O}(\kappa^{4}\epsilon^{-1})$}
\\ \cline{1-5} 
No-loop ITD (this paper)\cellcolor{blue!15}  &\cellcolor{blue!15} $\Theta(1)$ & \cellcolor{blue!15} $\mathcal{O}\Big(\frac{\kappa^3}{ K}  +  \kappa^3\Big)$ &\cellcolor{blue!15}
N/A  &\cellcolor{blue!15}
N/A 
\\ \cline{1-5}
Lower bound (this paper) \cellcolor{blue!15}&    \cellcolor{blue!15} $\Theta(1)$& \cellcolor{blue!15} $\Omega\big(\kappa^2\big)$  &\cellcolor{blue!15}
N/A  & \cellcolor{blue!15}
N/A 
\\ \hline
\end{tabular}
\vspace{-0.1cm}
\end{table*}

\begin{list}{$\bullet$}{\topsep=1ex \leftmargin=0.3in \rightmargin=0.2in \itemsep =0.2in}
 
\item The first contribution of this paper lies in the development of a unified convergence theory for AID-BiO, which is applicable to all choices of $N$ and $Q$. We further specialize our general theorems to provide the computational complexity for all of the above four implementations (as summarized in \Cref{tab:results}). Comparison among them suggests that AID-BiO does benefit from both $N$-loop and $Q$-loop. 
This is in contrast to minimax optimization (a special case of bilevel optimization), where it is shown in \cite{lin2020gradient,zhang2020single} that (No-loop) gradient descent ascent (GDA) with $N=1$ often outperforms ($N$-loop) GDA with $N=\kappa\ln \kappa$ (here $N$ denotes the number of ascent iterations for each descent iteration). To explain the reason, the gradient w.r.t.~$x$ in bilevel optimization involves additional second-order derivatives (that do not exist in minimax optimization), which are more sensitive to the accuracy of the optimal point of the inner function. Therefore, a large $N$ finds such a more accurate solution, and is hence more beneficial for bilevel optimization than minimax optimization.  
\end{list}
Differently from AID-BiO, the ITD-based bilevel approach (which we call as ITD-BiO) constructs the outer-level hypergradient estimation via backpropagation along the $N$-loop iteration path, and $Q=N$ always holds. Thus, ITD-BiO has only two implementation choices: $N$-$N$-loop (with large $N=\kappa\ln\kappa$) and No-loop (with small $N=\mathcal{O}(1)$). Here, $N$-$N$-loop and No-loop also refer to additional loops for solving sub-problems within the ITD-BiO's base loop of updating the variable $x$. The only convergence rate analysis on ITD-BiO was provided in \cite{ji2021bilevel} but only for $N$-$N$-loop, which does not suggest how $N$-$N$-loop compares with No-loop. It is still an open question whether ITD-BiO benefits from $N$-loops. 
\begin{list}{$\bullet$}{\topsep=1ex \leftmargin=0.3in \rightmargin=0.2in \itemsep =0.2in}

\item The second contribution of this paper lies in the development of a unified convergence theory for ITD-BiO, which is applicable to all values of $N$. We then specialize our general theorem to provide the computational complexity for both of the above implementations (as summarized in \Cref{tab:results_itd}). We further develop a convergence lower bound, which suggests that $N$-$N$-loop is necessary to guarantee a vanishing convergence error, whereas the no-loop scheme suffers from an unavoidable non-vanishing convergence error. 
\end{list}
The technical contribution of this paper is two-fold.  For AID methods, most existing studies including \cite{ji2021bilevel} solve the linear system 
with large {\small$Q=\Theta(\kappa\log\kappa)$} so that the upper-level Hessian-inverse-vector product approximation error can vanish. In contrast, we allow arbitrary (possibly small) {\small$Q$}, and hence this upper-level error can be large and nondecreasing, posing a key challenge to guarantee convergence. We come up with a novel idea to prove the convergence by showing that this error, not by itself but jointly with the inner-loop error, admits an (approximately) iteratively decreasing property, which bounds the hypergradient error and yields convergence. The analysis contains new developments to handle the coupling between this error and the inner-loop error, which is critical in our proof.
For ITD methods, unlike existing studies including \cite{ji2021bilevel}, we remove the boundedness assumption on {\small$y^*(x)$} via a novel error analysis over the entire execution rather than a single iteration.
Our analysis tools are general and can be extended to stochastic and acceleration bilevel optimizers.  

\begin{algorithm}[t]
	\caption{ AID-based bilevel optimization (AID-BiO) with double warm starts}   
	\label{alg:main_deter}
	\begin{algorithmic}[1]
		\STATE {\bfseries Input:}  Stepsizes $\alpha, \beta,\eta >0$, initializations $x_0, y_0,v_0$.
		\FOR{$k=0,1,2,...,K$}
		\STATE{Set $y_k^0 = y_{k-1}^{N} \mbox{ if }\; k> 0$ and $y_0$ otherwise  \textbf{\em (warm start initialization)}} 
		\FOR{$t=1,....,N$}
		\vspace{0.05cm}
		\STATE{Update $y_k^t = y_k^{t-1}-\alpha \nabla_y g(x_k,y_k^{t-1}) $}
		\vspace{0.05cm}
		\ENDFOR
                  \STATE{Hypergradient estimation via: 
                  	\vspace{0.05cm}
                  \\\hspace{0.1cm} Set $v_k^0 = v_{k-1}^{Q} \mbox{ if }\; k> 0$ and $v_0$ otherwise \textbf{\em (warm start initalization)}. 
                  	\vspace{0.05cm}
                  \\\hspace{0.1cm} Solve $v_k^Q$ from $\nabla_y^2 g(x_k,y_k^N) v = 
\nabla_y f(x_k,y^N_k)$ iteratively with $Q$ steps, stepsize $\eta$ and initialization  $v_k^0$
	\vspace{0.05cm}
	\vspace{0.05cm}
\\\hspace{0.1cm} Compute $\widehat\nabla \Phi(x_k)= \nabla_x f(x_k,y_k^N) -\nabla_x \nabla_y g(x_k,y_k^N)v_k^Q$ 
                    }
                 \STATE{Update $x_{k+1}=x_k- \beta \widehat\nabla \Phi(x_k) $}
		\ENDFOR
	\end{algorithmic}
	\end{algorithm}
\section{Algorithms}\label{sec:alg}
\subsection{AID-based Bilevel Optimization Algorithm}\label{sec:alg_aid}
As shown in~\Cref{alg:main_deter}, we present the general AID-based bilevel optimizer (which we refer to AID-BiO for short). At each iteration $k$ of the base loop,
AID-BiO first executes $N$ steps of gradient decent (GD) over the inner function $g(x,y)$ to find an approximation point $y_k^N$, where $N$ can be chosen either at a constant level or as large as $N=\kappa\ln\kappa$ (which forms an {\bf $N$-loop} of iterations). Moreover, to accelerate the practical training and achieve a stronger performance guarantee, AID-BiO often adopts a warm-start strategy by setting the initialization $y_{k}^0$ of each $N$-loop to be the output $y_{k-1}^N$ of the preceding $N$-loop rather than a random start.

To update the outer variable, AID-BiO adopts the gradient descent, by approximating the true gradient $\nabla \Phi(x_k)$ of the outer function w.r.t.\ $x$ (called hypergradient) that takes the following form:
\begin{align}\label{trueG}
\text{(True hypergradient:)} \quad \nabla \Phi(x_k) =&  \nabla_x f(x_k,y^*(x_k)) -\nabla_x \nabla_y g(x_k,y^*(x_k)) v_k^*, 
\end{align}
where $v_k^*$ is the solution of the linear system {\small$ \nabla_y^2 g(x_k,y^*(x_k))v=
\nabla_y f(x_k,y^*(x_k))$}. 
To approximate the above true hypergradient, AID-BiO first solves $v_k^Q$ as an approximate solution to a linear system {\small $\nabla_y^2 g(x_k,y_k^N) v = 
\nabla_y f(x_k,y^N_k)$} 
using $Q$ steps of GD iterations with stepsize $\eta$ starting from $v_k^0$. Here, $Q$ can also be chosen either at a constant level or as large as $Q=\kappa\ln \frac{\kappa}{\mu}$ (which forms a {\bf $Q$-loop} of iterations). Note that a warm start is also adopted here by setting $v_k^0=v_{k-1}^Q$, which is critical to achieve the convergence guarantee for small $Q$. If $Q$ is large enough, e.g., at an order of $\kappa\ln \frac{\kappa}{\epsilon}$, a zero initialization with $v_k^0=0$ suffices to solve the linear system well. 
Then, AID-BiO constructs a hypergradient estimator {\small $\widehat\nabla \Phi(x_k)$} given by 
\begin{align}\label{hyper-aid}
\text{(AID-based hypergradient estimate:)} \quad \widehat\nabla \Phi(x_k)= \nabla_x f(x_k,y_k^N) -\nabla_x \nabla_y g(x_k,y_k^N)v_k^Q.
\end{align}
Note that the execution of AID-BiO involves only Hessian-vector products in solving the linear system and Jacobian-vector product $\nabla_x \nabla_y g(x_k,y_k^N)v_k^Q$ which are more computationally tractable  than the calculation of second-order derivatives.

It is clear that different choices of $N$ and $Q$ lead to four implementations within the base loop of AID-BiO: $N$-loop (with large $N=\kappa\ln\kappa$ and small $Q=\mathcal{O}(1)$), $N$-$Q$-loop (with large $N=\kappa\ln\kappa$ and $Q=\kappa\ln\kappa$), $Q$-loop (with small $N=1$ and large $Q=\kappa\ln\kappa$) and No-loop (with small $N=1$ and $Q=\mathcal{O}(1)$).
In \Cref{sec:theory_aid}, we will establish a unified convergence theory for AID-BiO applicable to all its implementations in order to formally compare their computational efficiency. 

\subsection{ITD-Based Bilevel Optimization Algorithm}
As shown in~\Cref{alg:main_itd}, the ITD-based bilevel optimizer (which we refer to as ITD-BiO) updates the inner variable $y$ similarly to AID-BiO, and obtains the $N$-step output $y_k^N$ of GD with a warm-start initialization. ITD-BiO differentiates from AID-BiO mainly in its estimation of the hypergradient.
Without leveraging the implicit gradient formulation, ITD-BiO computes a direct derivative $\frac{\partial f(x_k,y^N_k)}{\partial x_k}$ via automatic differentiation for hypergradient approximation. Since $y^N_k$ has a dependence on $x_k$ through the $N$-loop iterative GD updates, the execution of ITD-BiO takes the backpropagation over the entire $N$-loop trajectory. To elaborate, it can be shown via the chain rule that the hypergradient estimate $\frac{\partial f(x_k,y^N_k)}{\partial x_k}$ takes the following  form of  
{\small$\frac{\partial f(x_k,y^N_k)}{\partial x_k}= \nabla_x f(x_k,y_k^N) -\alpha\sum_{t=0}^{N-1}\nabla_x\nabla_y g(x_k,y_k^{t})\prod_{j=t+1}^{N-1}(I-\alpha  \nabla^2_y g(x_k,y_k^{j}))\nabla_y f(x_k,y_k^N).$}
As shown in this equation, the differentiation does not compute the second-order derivatives directly but compute more tractable and economical  Hessian-vector products {\small$ \nabla^2_y g(x_k,y_k^{j-1})v_j, j=1,...,N$} (similarly for Jacobian-vector products), where each $v_j$ is  obtained recursively via    
$v_{j-1} = (I-\alpha  \nabla^2_y g(x_m,y_m^{j}))v_j\text{ with } v_N = \nabla_y f(x_m,y_m^N).$

Clearly, the implementation of ITD-BiO implies that $N=Q$ always holds. Hence, ITD-BiO takes only two possible architectures within its base loop: $N$-$N$-loop (with large $N=\kappa\ln\frac{\kappa}{\epsilon}$) and No-loop (with small $N=1$). In \Cref{sec:theory_itd}, we will establish a unified convergence theory for ITD-BiO applicable to both of its implementations in order to formally compare their computational efficiency.

\begin{algorithm}[t]
	\caption{ITD-based bilevel optimization algorithm (ITD-BiO) with warm start}   
	\label{alg:main_itd}
	\begin{algorithmic}[1]
		\STATE {\bfseries Input:}  Stepsize $\alpha>0$, initializations $x_0$ and $y_0$ .
		\FOR{$k=0,1,2,...,K$}
		\STATE{Set $y_k^0 = y_{k-1}^{N} \mbox{ if }\; k> 0$ and $y_0$ otherwise \textbf{\em (warm start initialization)} }
		\FOR{$t=1,....,N$}
		\vspace{0.05cm}
		\STATE{Update $y_k^t = y_k^{t-1}-\alpha \nabla_y g(x_k,y_k^{t-1}) $}
		\vspace{0.05cm}
		\ENDFOR
                  \STATE{Compute $\widehat\nabla \Phi(x_k)=\frac{\partial f(x_k,y^N_k)}{x_k}$ via backpropagation w.r.t. $x_k$ 
                    }
                 \STATE{Update $x_{k+1}=x_k- \beta \widehat\nabla \Phi(x_k) $}
		\ENDFOR
	\end{algorithmic}
	\end{algorithm}

\section{Definitions and Assumptions}\label{sec:def}
This paper focuses on the following types of objective functions.
\begin{assum}\label{assum:geo}
The inner-level function $g(x,y)$ is $\mu$-strongly-convex w.r.t.~$y$.
\end{assum}
Since the objective function $\Phi(x)$ in \cref{objective_deter} is possibly nonconvex, algorithms are expected to find an $\epsilon$-accurate stationary point defined as follows. 
\begin{definition}
We say $\bar x$ is an $\epsilon$-accurate stationary point for 
the bilevel optimization problem given in \cref{objective_deter}
if $\|\nabla \Phi(\bar x)\|^2\leq \epsilon$, where $\bar x$ is the output of  an algorithm.
\end{definition}
In order to compare the performance of different bilevel algorithms, we adopt the following metrics of computational complexity. 
\begin{definition}\label{com_measure}
Let $\mbox{\normalfont Gc}(\epsilon)$ be the number of gradient evaluations, and $\mbox{\normalfont MV}(\epsilon)$ be the total number of Jacobian- and Hession-vector product evaluations to achieve an $\epsilon$-accurate stationary point of the bilevel optimization problem in \cref{objective_deter}.
 \end{definition} 
Let $z=(x,y)$. We take the following standard assumptions,
as also widely adopted by~\cite{ghadimi2018approximation,ji2020convergence}.
\begin{assum}\label{ass:lip}
Gradients $\nabla f(z)$ and $\nabla g(z)$ are $L$-Lipschitz, i.e., for any $z,z^\prime$, $$\|\nabla f(z)-\nabla f(z^\prime)\|\leq L\|z-z^\prime\|,\quad \|\nabla g(z)-\nabla g(z^\prime)\|\leq L\|z-z^\prime\|.$$
\end{assum}
As shown in~\cref{trueG}, the gradient of the objective function $\Phi(x)$ involves the second-order derivatives $\nabla_x\nabla_y g(z)$ and $\nabla_y^2 g(z)$. The following assumption imposes the Lipschitz conditions on such higher-order derivatives, as also made in~\cite{ghadimi2018approximation}.
\begin{assum}\label{high_lip}
Suppose the derivatives $\nabla_x\nabla_y g(z)$ and $\nabla_y^2 g(z)$ are $\rho$-Lipschitz, i.e., for any $z,z^\prime$
$$\|\nabla_x\nabla_y g(z)-\nabla_x\nabla_y g(z^\prime)\| \leq \rho \|z-z^\prime\|, \quad \|\nabla_y^2 g(z)-\nabla_y^2 g(z^\prime)\|\leq \rho \|z-z^\prime\|.$$
\end{assum}
To guarantee the boundedness the hypergradient estimation error, existing works~\citep{ghadimi2018approximation,ji2020convergence,grazzi2020iteration} assume that the gradient $\nabla f(z) $ is bounded for all $z=(x,y)$. Instead, we make a weaker boundedness assumption on the gradients $\nabla_y f(x,y^*(x))$. 
\begin{assum}\label{ass:boundGradient}
There exists a constant $M$ such that for any $x$, $\|\nabla_y f(x,y^*(x))\|\leq M$.
\end{assum}
For the case where the total objective function $\Phi(\cdot)$ has some benign structures, e.g., convexity or strong convexity, Assumption~\ref{ass:boundGradient} can be removed by an induction analysis that all iterates are bounded as  in~\cite{ji2021lower}. Assumption~\ref{ass:boundGradient} can also be removed by projecting $x$ onto a bounded constraint set $\mathcal{X}$. 

\section{Convergence Analysis of AID-BiO }\label{sec:theory_aid}

As we describe in \Cref{sec:alg_aid}, AID-BiO can have four possible implementations depending on whether $N$ and $Q$ are chosen to be large enough to form an $N$-loop and/or $Q$-loop. 
In this section, we will provide the convergence analysis and characterize the overall computational complexity for all of the four implementations, which will provide the general guidance on which algorithmic architecture is computationally most efficient. 

\subsection{Convergence Rate and Computational Complexity}




In this subsection, we develop two unified theorems for AID-BiO, both of which are applicable to all the regimes of $N$ and $Q$. 
We then specialize these theorems to provide the complexity bounds (as corollaries) for the four implementations of AID-BiO. It turns out that the first theorem provides tighter complexity bounds for the implementations with small $Q=\Theta(1)$, and the second theorem provides tighter complexity bounds for the implementations with large $Q=\kappa\ln\frac{\kappa}{\epsilon}$. Our presentation of those corollaries below will thus focus only on the tighter bounds.  
The following theorem provides our first unified convergence analysis for AID-BiO.
\begin{theorem}\label{th:mainconverge1}
Suppose Assumptions~\ref{assum:geo}, \ref{ass:lip}, \ref{high_lip} and \ref{ass:boundGradient} hold. Choose parameters $\alpha,\eta$ and $\lambda$ such that $(1+\lambda)(1-\alpha\mu)^N(1+4r(1+\frac{1}{\eta\mu})L^2)\leq 1-\eta\mu$, where $r=\frac{C_Q^2}{(\frac{\rho M}{\mu}+L)^2}$ with {\small$C_Q = \frac{Q(1-\eta\mu)^{Q-1}\rho M\eta}{\mu} + \frac{1-(1-\eta\mu)^Q(1+\eta Q\mu)}{\mu^2}\rho M + (1-(1-\eta\mu)^Q) \frac{L}{\mu}$}. 
Let $L_\Phi= L + \frac{2L^2+\rho M^2}{\mu} + \frac{2\rho L M+L^3}{\mu^2} + \frac{\rho L^2 M}{\mu^3}$ be the smoothness parameter of $\Phi(\cdot)$. Let $\widetilde w:=\frac{(1-\eta\mu)\eta\mu}{3\lambda r L^2}\big(  1+ \frac{\rho^2M^2}{L^2\mu^2}\big)\frac{L^2}{\mu^2}+ \big( 1+\frac{1}{\eta\mu}\big)\big( L^2 + \frac{\rho^2M^2}{\mu^2}\big)\big(\frac{16(1-\eta\mu)^{2Q}}{\mu^2} + \frac{4(1-\eta\mu)\eta\mu}{3\lambda L^2}\big) \frac{L^2}{\mu^2}$.
Choose the outer stepsize $\beta$ such that 
$\beta =\min\big\{ \frac{1}{12L_\Phi},\,\sqrt{ \frac{\eta\mu}{18L^2\widetilde w}}\big\}.$
Then, 
\begin{align}
 \frac{1}{K} \sum_{k=0}^{K-1}\|\nabla\Phi(x_k)\|^2 \leq \frac{8(\Phi(x_0)- \Phi(x^*))}{\beta K}  + \frac{21L^2((1+ \frac{\rho^2 M^2}{L^2\mu^2}  )\|y_0^*\|^2+(\frac{3M}{\mu} + \frac{2L}{\mu}\|y_0^*\|)^2)}{\eta\mu K}.
\end{align}
\end{theorem}
\Cref{th:mainconverge1} also elaborates the precise requirements on the stepsizes $\alpha$, $\eta$ and $\beta$ and the auxiliary parameter $\lambda$, which take complicated forms. In the following, by further specifying these parameters,
we characterize the complexities for AID-BiO in more explicit forms. We focus on the implementations with $Q=\Theta(1)$ (for which \Cref{th:mainconverge1} specializes to tighter bound than  \Cref{th:wodetiantiannass} below), which includes the $N$-loop scheme (with $N=\Theta(\kappa\ln\kappa)$) and the No-loop scheme (with $N=1$). 
\begin{corollary}[$N$-loop]\label{co:coresulstscq1}
Consider $N$-loop AID-BiO with $N=\Theta(\kappa\ln\kappa)$ and $Q=\Theta(1)$, where $\kappa = \frac{L}{\mu}$ denotes the condition number of the inner problem. Under the same setting of \Cref{th:mainconverge1},
choose  $\eta=\frac{1}{L}$, $\alpha = \frac{1}{L}$, and $\lambda =1$.
Then, we have 
{\small $\frac{1}{K}\sum_{k=0}^{K-1}\|\nabla\Phi(x_k)\|^2 =\mathcal{O}\big( \frac{\kappa^4}{K} + \frac{\kappa^3}{K}\big)$},
and the complexity to achieve an $\epsilon$-accurate stationary point is  {\small$\mbox{\normalfont Gc}(\epsilon)=\mathcal{\widetilde O}(\kappa^5\epsilon^{-1}),\mbox{\normalfont MV}(\epsilon)=\mathcal{\widetilde O}\left(\kappa^4\epsilon^{-1}\right)$}.
\end{corollary}

\begin{corollary}[No-loop]\label{co:coresult2}
Consider No-loop AID-BiO with $N=1$ and  $Q=\Theta(1)$. Under the same setting of \Cref{th:mainconverge1},  
choose parameters $\alpha=\frac{1}{L}$, $\lambda = \frac{\alpha\mu}{2}$ and $\eta=\min\{\frac{1}{128}\frac{\alpha\mu^2}{Q^2L^2},\frac{\alpha}{4},\frac{1}{\mu Q}\}$.
Then,  {\small$\frac{1}{K}\sum_{k=0}^{K-1}\|\nabla\Phi(x_k)\|^2 =\mathcal{O}\big(\frac{\kappa^6}{K} + \frac{\kappa^5}{K}\big)$}, and the complexity is {\small$\mbox{\normalfont Gc}(\epsilon)=\mathcal{\widetilde O}(\kappa^6\epsilon^{-1}),\mbox{\normalfont MV}(\epsilon)=\mathcal{\widetilde O}(\kappa^6\epsilon^{-1})$}.
\end{corollary}
The analysis of \Cref{th:mainconverge1} can be further improved for the large $Q$ regime, which guarantees a sufficiently small outer-level approximation error, and helps to relax the requirement on the stepsize $\eta$. Such an adaptation yields the following alternative unified convergence characterization for AID-BiO, which is applicable for all $Q$ and $N$, but specializes to tighter complexity bounds than \Cref{th:mainconverge1} in the large $Q$ regime.
For simplicity, we set the initialization $v_k^0=0$ in \Cref{alg:main_deter}. 

\begin{theorem}\label{th:wodetiantiannass}
Suppose Assumptions~\ref{assum:geo}, \ref{ass:lip}, \ref{high_lip} and \ref{ass:boundGradient} hold. Define 
{\small$\tau =  (1-\alpha\mu)^N(1+\lambda + 6(1+\lambda^{-1})( L^2 +\rho^2 M^2\mu^{-2} +2L^2C_Q^2 \big)L^2\beta^2\mu^{-2}),\;w =  6(1-\alpha\mu)^N( L^2 +\rho^2 M^2\mu^{-2} +2L^2C_Q^2 )(1+\lambda^{-1})L^2\mu^{-2}$}, where $C_Q$ is a positive constant defined as in \Cref{th:mainconverge1}. 
Choose parameters {\small$\alpha, \beta$} such that $\tau<1$ and {\small$\beta L_\Phi+w\beta^2\big(\frac{1}{2} + \beta L_\Phi\big)\frac{1}{1-\tau} \leq \frac{1}{4}$} hold. Then, the output of AID-BiO satisfies
\begin{align*}
\frac{1}{K}\sum_{k=0}^{K-1} \|\nabla\Phi(x_k) \|^2 \leq & \frac{4(\Phi(x_0)-\Phi(x^*))}{\beta K} + \frac{3}{K} \frac{\delta_0}{1-\tau}+ \frac{27L^2M^2}{\mu^2}(1-\eta\mu)^{2Q},
\end{align*} 
where {\small $\delta_0 = 3 \big( L^2 +\frac{\rho^2 M^2}{\mu^2} +2L^2C_Q^2 \big)(1-\alpha\mu)^N \|y_0^* - y_0 \|^2$} is the initial distance.
\end{theorem}
We next specialize \Cref{th:wodetiantiannass} to obtain the complexity for two implementations of AID-BiO with $Q=\Theta(\kappa\ln\kappa)$: $N$-$Q$-loop (with {\small $N=\Theta(\kappa\ln\kappa)$}) and $Q$-loop (with {\small$N=1$}), as shown in the following two corollaries. For each case, we need to set the parameters $\lambda,\eta$ and $\alpha$ in \Cref{th:wodetiantiannass} properly.



\begin{corollary}[$N$-$Q$-loop]\label{co:xindelargeNlargeQ}
Consider $N$-$Q$-loop AID-BiO with {\small$N=\Theta(\kappa\ln \kappa)$} and {\small$Q=\Theta(\kappa \ln \frac{\kappa}{\epsilon})$}. Under the same setting of \Cref{th:wodetiantiannass}, choose 
$\eta=\alpha = \frac{1}{L}$, $\lambda =1$ and $\beta =\Theta(\kappa^{-3})$. 
Then, {\small$\frac{1}{K}\sum_{k=0}^{K-1}\|\nabla\Phi(x_k)\|^2 =\mathcal{O}\big( \frac{\kappa^3}{K} + \epsilon\big)$}, and the complexity  is {\small$\mbox{\normalfont Gc}(\epsilon)=\mathcal{\widetilde O}(\kappa^4\epsilon^{-1})$, $ \mbox{\normalfont MV}(\epsilon)=\mathcal{\widetilde O}(\kappa^4\epsilon^{-1})$}.
\end{corollary}  
\begin{corollary}[$Q$-loop]\label{co:singQnn1}
Consider $Q$-loop AID-BiO with $N=1$ and {\small$Q=\Theta(\kappa\ln\frac{\kappa}{\epsilon})$}. Under the same setting of \Cref{th:wodetiantiannass}, choose  $\alpha=\eta=\frac{1}{L}$, $\lambda = \frac{\alpha\mu}{2}$ and $\beta=\Theta(\kappa^{-4})$.  
Then, {\small$\frac{1}{K}\sum_{k=0}^{K-1}\|\nabla\Phi(x_k)\|^2 =\mathcal{O}\big( \frac{\kappa^5}{K} +\frac{\kappa^4}{K} +\epsilon\big)$}, and the complexity is {\small$\mbox{\normalfont Gc}(\epsilon)=\mathcal{\widetilde O}(\kappa^5\epsilon^{-1})$, $\mbox{\normalfont MV}(\epsilon)=\mathcal{\widetilde O}(\kappa^{6}\epsilon^{-1})$}.
\end{corollary}

\vspace{0.2cm}
\noindent{\bf Discussion on hyperparameter selection for different implementations.} 
For all loop-sizes, we set the hyperparameters to achieve the best complexity as long as convergence is guaranteed. 
Let us elaborate on 
$N$-loop (\Cref{co:coresulstscq1}) and No-loop (\Cref{co:coresult2}). 
{\bf At a proof level}, $\lambda$ needs to satisfy {\small$(1-\alpha\mu)^N(1+\lambda)<1$} (see \Cref{le:yknstart}) to guarantee the convergence; otherwise the inner-loop error will explode.
Given this requirement, for $N$-loop with {\small$N=\Theta(\kappa\log\kappa)$}, {\small$\lambda=\Theta(1)$} achieves the best complexity. However, for No-loop with {\small $N=1$}, the requirement becomes {\small$(1-\alpha\mu)(1+\lambda)<1$}, and {\small$\lambda=\Theta(\mu)$} achieves the best complexity. The stepsize $\eta$ appears in {\small$(1-\alpha\mu)^N\frac{\eta}{\mu}\|y_{k-1}^N-y_{k-1}^* \|^2)$} (see \Cref{le:vqstart}) of  the error {\small$\|v_k^Q-v_k^*\|^2$}.
Given the requirement {\small$(1-\alpha\mu)^N\frac{\eta}{\mu}<1$}, for $N$-loop with  {\small$N=\Theta(\kappa\log\kappa)$}, {\small$\eta = \Theta(1)$} achieves the best complexity, whereas for No-loop with {\small$N=1$}, the best {\small$\eta=\Theta(\mu)$}.  {\bf At a conceptual level}, estimating the hypergradient and linear system contains the inner-loop error {\small$\|y_{k}^N-y_{k}^* \|^2$}. For {\small$N=1$}, the per-iteration error is large, and hence we need smaller stepsizes $\lambda,\eta,\beta$ to ensure the accumulated error not to explode. A similar argument holds for $N$-$Q$-loop and $Q$-loop.

\subsection{Comparison among Four Implementations}

\vspace{0.2cm}
{\noindent \bf Impact of $N$-loop ($N=1$ vs $N=\kappa\ln \kappa$).} We fix $Q$, and compare how the choice of $N$ affects the computational complexity. First, let $Q=\Theta(1)$, and compare the results between the two implementations $N$-loop with $\Theta(\kappa\ln\kappa)$ (\Cref{co:coresulstscq1}) and No-loop with $N=1$  (\Cref{co:coresult2}). Clearly, the $N$-loop scheme significantly improves the convergence rate of the No-loop scheme from $\mathcal{O}(\frac{\kappa^6}{K})$ to $\mathcal{O}(\frac{\kappa^4}{K})$, and improves the matrix-vector and gradient complexities from $\mathcal{\widetilde O}(\kappa^6\epsilon^{-1})$ and  $\mathcal{\widetilde O}(\kappa^6\epsilon^{-1})$ to $\mathcal{\widetilde O}(\kappa^4\epsilon^{-1})$ and  $\mathcal{\widetilde O}(\kappa^5\epsilon^{-1})$, respectively. To explain intuitively, the hypergradient estimation involves a coupled error {\small$\eta \|y_k^N-y^*(x_k)\|$} induced from solving the linear system {\small $\nabla_y^2 g(x_k,y_k^N) v = 
\nabla_y f(x_k,y^N_k)$} with stepsize $\eta$. Therefore, a smaller inner-level approximation error {\small$\|y_k^N-y^*(x_k)\|$} allows a more aggressive stepsize $\eta$, and hence yields a faster convergence rate as well as a lower total complexity, as also demonstrated in our experiments. It is worth noting that such a comparison is generally different from that in minimax optimization~\citep{lin2020gradient,zhang2020single}, where alternative (i.e., No-loop) gradient descent ascent (GDA) with $N=1$ outperforms (N-loop) GDA with $N=\kappa\ln \kappa$, where $N$ denotes the number of ascent iterations for each descent iteration. To explain the reason, in constrast to minimax optimization, the gradient w.r.t.~$x$ in bilevel optimization involves {\bf additional} second-order derivatives, which are more sensitive to the inner-level approximation error. 
Therefore, a larger $N$ is more beneficial for bilevel optimization than minimax optimization. 
Similarly, we can also fix $Q=\Theta(\kappa\ln\kappa)$, the $N$-$Q$-loop scheme with $N=\kappa\ln \kappa$ (\Cref{co:xindelargeNlargeQ}) significantly outperforms the $Q$-loop scheme with $N=1$ (\Cref{co:singQnn1}) in terms of the convergence rate and complexity. 



\vspace{0.2cm}

{\noindent \bf Impact of $Q$-loop ($Q=1$ vs $Q=\Theta(\kappa \ln \frac{\kappa}{\epsilon})$).} We fix $N$, and characterize the impact of the choice of $Q$ on the  complexity. For $N=1$, comparing No-loop with $Q=\Theta(1)$ in \Cref{co:coresult2} and $Q$-loop with $Q=\Theta(\kappa\ln\kappa)$ in \Cref{co:singQnn1} shows that both choices of $Q$ yield the same matrix-vector complexity $\mathcal{\widetilde O}(\kappa^6\epsilon^{-1})$, but $Q$-loop with a larger $Q$ improves the gradient complexity of No-loop with $Q=\Theta(1)$ from $\mathcal{\widetilde O}(\kappa^6\epsilon^{-1})$ to $\mathcal{\widetilde O}(\kappa^5\epsilon^{-1})$. A similar phenomenon can be observed for $N=\Theta(\kappa\ln\kappa)$ based on the comparision between $N$-$Q$-loop in \Cref{co:xindelargeNlargeQ} and $N$-loop in \Cref{co:coresulstscq1}. 

\vspace{0.2cm}

\noindent {\bf In deep learning.} Also note that in the setting where the matrix-vector complexity dominates the gradient complexity, e.g., in deep learning, 
such two choices of $Q$ do not affect the total computational complexity. However, a smaller $Q$ can help reduce the per-iteration load on the computational resource and memory, and hence is preferred in practical applications with large models.

\vspace{0.2cm}

{\noindent \bf Comparison among four implementations.} By comparing the complexity results in Corollaries~\ref{co:coresulstscq1},~\ref{co:coresult2}, \ref{co:xindelargeNlargeQ} and \ref{co:singQnn1}, it can be seen that $N$-$Q$-loop and $N$-loop (both with a large $N=\Theta(\kappa\ln\kappa)$) achieve the best matrix-vector complexity $\mathcal{\widetilde O}(\kappa^4\epsilon^{-1})$, whereas $Q$-loop and No-loop (both with a smaller $N=1$) require higher matrix-vector complexity of $\mathcal{\widetilde O}(\kappa^6\epsilon^{-1})$. Also note that 
$N$-$Q$-loop has the lowest gradient complexity. 
This suggests that the introduction of the inner loop with large $N$ can help to reduce the total computational complexity.       




\vspace{-0.2cm}
\section{Convergence Analysis of ITD-BiO}\label{sec:theory_itd}
\vspace{-0.1cm}
In this section, we first provide a unified theory for ITD-BiO, which is applicable for all choices of $N$, and then specialize the convergence theory to characterize the computational complexity for the two implementations of ITD-BiO: No loop and $N$-$N$-loop. We also provide a convergence lower bound to justify the necessity of choosing large $N$ to achieve a vanishing convergence error. 
The following theorem characterizes the convergence rate of ITD-BiO for all choices of $N$.
\begin{theorem}\label{th:geiwogeofferbossc}  
Suppose Assumptions~\ref{assum:geo}, \ref{ass:lip}, \ref{high_lip} and \ref{ass:boundGradient} hold. Define {\small$w =\big( 1+\frac{2}{\alpha\mu}\big)\frac{L^2}{\mu^2}(1-\alpha\mu)^N\lambda_N + \frac{4M^2w_N^2L^2}{\mu^2}$} and {\small$\tau =N^2(1-\alpha\mu)^N+w_N^2 +\lambda_N(1-\alpha\mu)^N $}, where $\lambda_N$  and $w_N$ are given by 
{\small$\lambda_N =\frac{4M^2w_N^2+4(1-\frac{1}{4}\alpha\mu)L^2(1+\alpha LN)^2}{1-\frac{1}{4}\alpha\mu-(1-\alpha\mu)^N(1+\frac{1}{2}\alpha\mu)}, 
w_N = \alpha\Big(\rho+\frac{\alpha\rho L(1-(1-\alpha\mu)^{\frac{N}{2}})}{1-\sqrt{1-\alpha\mu}}\Big)  (1-\alpha\mu)^{\frac{N}{2}-1}\frac{1-(1-\alpha\mu)^{\frac{N}{2}}}{1-\sqrt{1-\alpha\mu}}.$ }
Choose parameters such that {\small$\beta^2\leq \frac{1-\frac{1}{4}\alpha\mu}{2w}, \alpha\leq\frac{1}{2L}$} and {\small$ \beta L_\Phi + \frac{8}{\alpha\mu}\Big(\frac{1}{2}+ \beta L_\Phi\Big)w\beta^2<\frac{1}{4}$}, where $L_\Phi= L + \frac{2L^2+\rho M^2}{\mu} + \frac{2\rho L M+L^3}{\mu^2} + \frac{\rho L^2 M}{\mu^3}$ denotes the smoothness parameter of $\Phi(\cdot)$. Then, we have
\begin{align}\label{eq:ggsmidadsadacas}
\frac{1}{K}\sum_{k=0}^{K-1}\|\nabla \Phi(x_k)\|^2 
\leq &\mathcal{O}\Big(\frac{\Delta_\Phi}{\beta K}  + \frac{\tau \Delta_y}{\mu^2K}+ \frac{(1-\alpha\mu)^{2N}}{\mu^3K}+ \frac{M^2\big(1-\alpha\mu\big)^{2N}L^2}{\alpha\mu^3}\Big),
\end{align}
where $\Delta_\Phi = \Phi(x_0)-\min_x\Phi(x)$ and $\Delta_y = \|y_0-y^*(x_0)\|^2$.
\end{theorem}
In \Cref{th:geiwogeofferbossc}, the upper bound on the convergence rate for ITD-BiO contains a convergent term $\mathcal{O}(\frac{1}{K})$ (which converges to zero sublinearly with $K$) and an error term $\mathcal{O}\big(\frac{M^2(1-\alpha\mu)^{2N}}{\alpha\mu^3}\big)$ (which is independent of $K$, and possibly non-vanishing if $N$ is chosen to be small). To show that such a possibly non-vanishing error term (when $N$ is chosen to be small) fundamentally exists, we next provide the following lower bound on the convergence rate of ITD-BiO. 
\begin{theorem}[{\bf Lower Bound}]\label{th:lowerBoundsacasqw}
Consider the ITD-BiO algorithm in \Cref{alg:main_itd} with $\alpha\leq \frac{1}{L}$, $\beta\leq \frac{1}{L_\Phi}$ and $N\leq \mathcal{O}(1)$, where $L_\Phi$ is the smoothness parameter of $\Phi(x)$. 
There exist objective functions $f(x,y)$ and $g(x,y)$ that satisfy Assumptions~\ref{assum:geo}, \ref{ass:lip}, \ref{high_lip} and \ref{ass:boundGradient} such that for all iterates $x_K$ (where $K\geq 1$) generated by ITD-BiO in \Cref{alg:main_itd}, 
{\small$\|\nabla\Phi(x_K)\|^2 \geq \Theta\big(\frac{L^2M^2}{\mu^2}\big(1-\alpha\mu\big)^{2N}\big).$}
\end{theorem}
Clearly, the error term in the upper bound given in \Cref{th:geiwogeofferbossc} matches the lower bound given in \Cref{th:lowerBoundsacasqw} in terms of $\frac{M^2L^2}{\mu^2}(1-\alpha\mu)^{2N}$, and there is still a gap on the order of $\alpha\mu$, which requires future efforts to address. \Cref{th:geiwogeofferbossc} and \Cref{th:lowerBoundsacasqw} together indicate that in order to achieve an $\epsilon$-accurate stationary point, $N$ has to be chosen as large as $N=\Theta(\kappa\log \frac{\kappa}{\epsilon})$. This corresponds to the $N$-$N$-loop implementation of ITD-BiO, where large $N$ achieves a highly accurate hypergradient estimation in each step. Another No-loop implementation chooses a small constant-level $N=\Theta(1)$ to achieve an efficient execution per step,  
where a large $N$ can cause large memory usage and computation cost. Following from \Cref{th:geiwogeofferbossc} and \Cref{th:lowerBoundsacasqw}, such No-loop implementation necessarily suffers from a non-vanishing error.

In the following corollaries, we further specialize \Cref{th:geiwogeofferbossc} to obtain the complexity analysis for ITD-BiO under the two aforementioned implementations of ITD-BiO. 
\begin{corollary}[$N$-$N$-loop]\label{co:itdwithlargen} Consider $N$-$N$-loop ITD-BiO with $N=\Theta(\kappa\ln\frac{\kappa}{\epsilon})$.
Under the same setting of \Cref{th:geiwogeofferbossc}, choose {\small$\beta = \min\Big\{\sqrt{\frac{\alpha\mu}{40 w}},\sqrt{\frac{1-\frac{\alpha\mu}{4}}{2w}},\frac{1}{8L_\Phi} \Big\}$}, $\alpha=\frac{1}{2L}$.  
Then, $\frac{1}{K}\sum_{k=0}^{K-1}\|\nabla \Phi(x_k)\|^2  = \mathcal{O}\big( \frac{\kappa^3}{K} +\epsilon\big)$,
and the complexity is $\mbox{\normalfont Gc}(\epsilon)=\mathcal{\widetilde O}(\kappa^4\epsilon^{-1})$, $ \mbox{\normalfont MV}(\epsilon)=\mathcal{\widetilde O}(\kappa^4\epsilon^{-1})$.
\end{corollary}
\Cref{co:itdwithlargen} shows that for a large $N=\Theta(\kappa\ln \frac{\kappa}{\epsilon})$, we can guarantee that ITD-BiO converges to an $\epsilon$-accurate stationary point, and the gradient and matrix-vector product complexities are given by $\mathcal{\widetilde O}(\kappa^4\epsilon^{-1})$. We note that \cite{ji2021bilevel} also analyzed the ITD-BiO with $N=\Theta(\kappa\ln \frac{\kappa}{\epsilon})$, and provided the same complexities as our results in \Cref{co:itdwithlargen}.  In comparison, our analysis has several differences. First, \cite{ji2021bilevel} assumed that the minimizer $y^*(x_k)$ at the $k^{th}$ iteration is bounded, whereas our analysis does not impose this assumption. Second, 
 \cite{ji2021bilevel} involved an {\bf additional} error term $\max_{k=1,...,K}\|y^*(x_k)\|\frac{L^2M^2(1-\alpha\mu)^N}{\mu^4}$, which can be very large (or even unbounded) under standard Assumptions~\ref{assum:geo}, \ref{ass:lip}, \ref{high_lip} and \ref{ass:boundGradient}.  
  We next characterize the convergence for the small $N=\Theta(1)$. 
\begin{corollary}[No-loop]\label{co:itdwithsmalln}
Consider No-loop ITD-BiO with $N=\Theta(1)$. Under the same setting of \Cref{th:geiwogeofferbossc}, choose stepsizes $\alpha=\frac{1}{2NL}$ and $\beta = \min\big\{\sqrt{\frac{\alpha\mu}{40 w}},\sqrt{\frac{1-\frac{\alpha\mu}{4}}{2w}},\frac{1}{8L_\Phi} \big\}$. Then, we have $\frac{1}{K}\sum_{k=0}^{K-1}\|\nabla \Phi(x_k)\|^2 = \mathcal{O}\big(\frac{\kappa^3}{ K}  +  \frac{M^2L^2}{\alpha\mu^3}\big)$.
\end{corollary}
\Cref{co:itdwithsmalln} indicates that for the constant-level $N=\Theta(1)$, the convergence bound contains a non-vanishing error $ \mathcal{O}(\frac{M^2L^2}{\alpha\mu^3})$. As shown in the convergence lower bound in \Cref{th:lowerBoundsacasqw}, under standard Assumptions~\ref{assum:geo}, \ref{ass:lip}, \ref{high_lip} and \ref{ass:boundGradient}, such an error is unavoidable. Comparison between the above two corollaries suggests that for ITD-BiO, the $N$-$N$-loop is necessary to guarantee a vanishing convergence error, whereas No-loop necessarily suffers from a non-vanishing convergence error.

\vspace{0.2cm}
\noindent{\bf Discussion on the setting with small response Jacobian.} Our results in \Cref{th:geiwogeofferbossc} and \Cref{th:lowerBoundsacasqw} apply to the general functions whose first- and second-order derivatives are Lipschitz continuous, i.e., under Assumptions~\ref{ass:lip} and~\ref{high_lip}. Here, we further discuss the extension of our results to another setting where the response Jacobian is extremely small. This setting occurs in some deep learning applications~\citep{finn2017model,ji2020convergence}, where the response Jacobian $\frac{\partial y^*(x)}{\partial x}$ (which is estimated by $\frac{\partial y_k^N(x)}{\partial x}$ with a large $N$) can be order-of-magnitude smaller than network gradients. Based on \cref{eq:aisstancelemmas} and \cref{worst_case_instance} in the appendix, it can be shown that the convergence error is proportional to the quantity {\small$\frac{1}{K}\sum_{k=0}^{K-1}\| \frac{\partial y^*(x_{k})}{\partial x_{k}} \|^2$}, and hence the constant-level $N=\Theta(1)$ can still achieve a small error in this setting.


\section{Empirical Verification}
{\bf Experiments on AID-BiO.} We first conduct experiments to verify our theoretical results in Corollaries~\ref{co:coresulstscq1},~\ref{co:coresult2}, \ref{co:xindelargeNlargeQ} and \ref{co:singQnn1} on AID-BiO with different implementations. We consider the following hyperparameter optimization problem.
\begin{equation*}
\begin{aligned}
    &\min_{\lambda}\mathcal{L}_{\mathcal{D}_{\text{val}}}(\lambda)=\frac{1}{|\mathcal{D}_{\text{val}}|}\sum_{\xi\in\mathcal{D}_{\text{val}}}\mathcal{L}(w_*;\xi),\;\;
    \text{s.t.}\; w^*=\argmin_{w}\frac{1}{|\mathcal{D}_{\text{tr}}|}\sum_{\xi\in\mathcal{D}_{\text{tr}}}\Big(\mathcal{L}(w;\xi)+\frac{\lambda}{2}\|w\|_2^2\Big),
\end{aligned}
\end{equation*}
where $\mathcal{D}_{\text{tr}}$ and $\mathcal{D}_{\text{val}}$ stand for training and validation datasets, $\mathcal{L}(w;\xi)$ denotes the loss function induced by the model parameter $w$ and sample $\xi$, and $\lambda>0$ denotes the regularization parameter. The goal is to find a good hyperparameter $\lambda$ to minimize the validation loss evaluated at the optimal model parameters for the regularized empirical risk minimization problem.
\begin{figure}[ht]
\vspace{-0.1cm}
    \centering
    \includegraphics[scale=0.3]{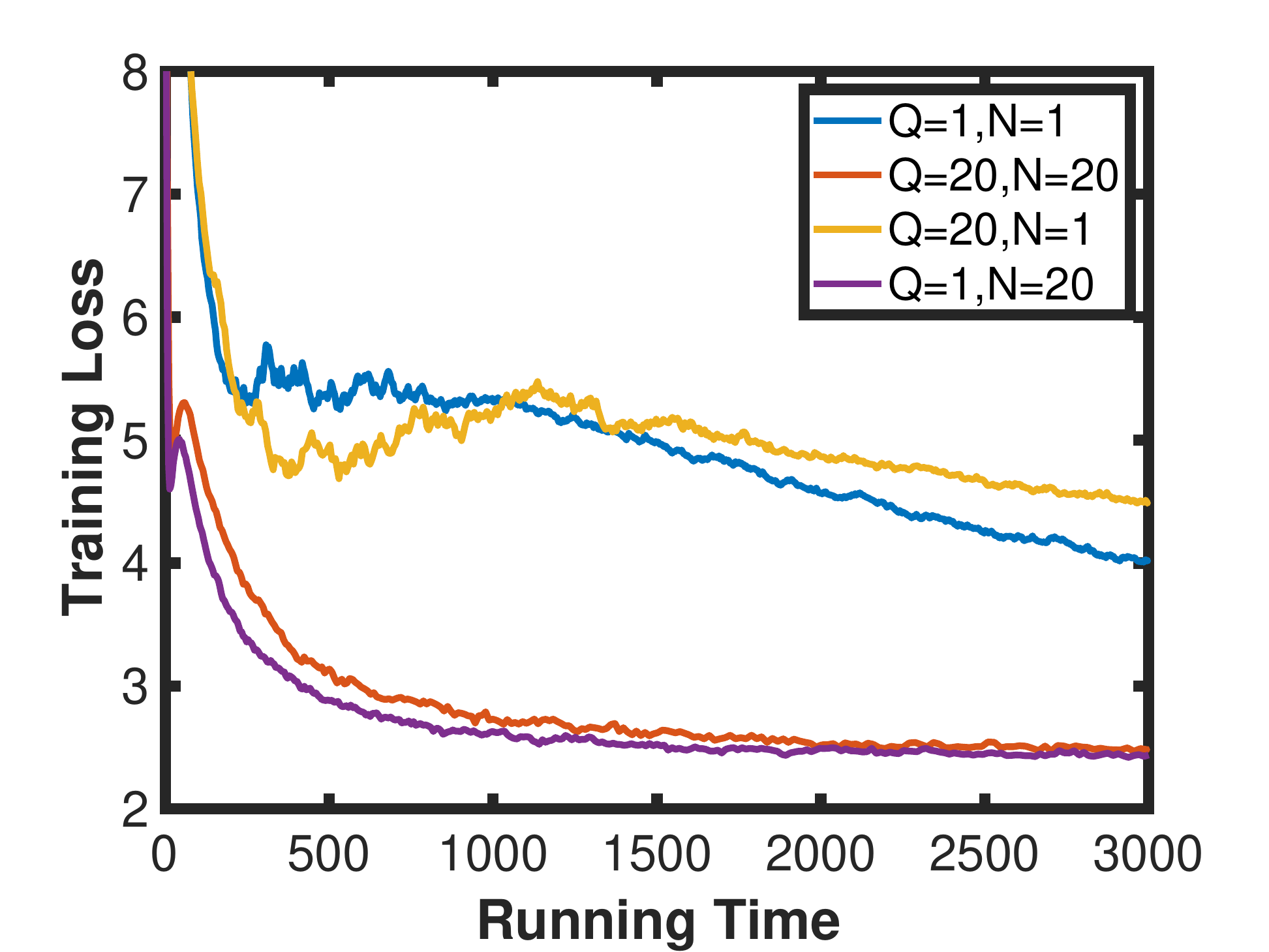}
    \includegraphics[scale=0.3]{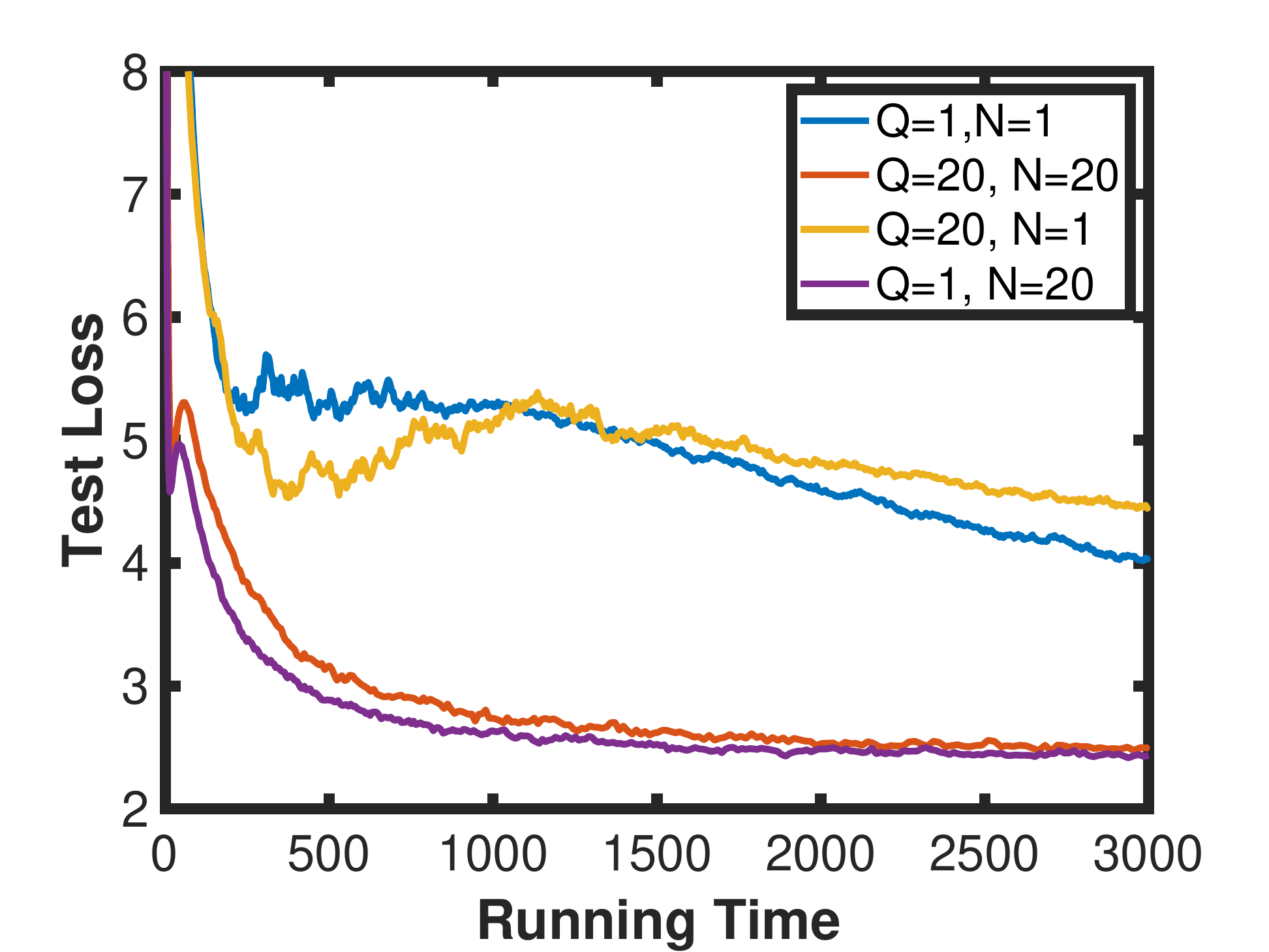}
    \caption{Training and test losses v.s.~running time (seconds) on MNIST  with different $Q$ and $N$.}
    \label{fig:experiment}
\end{figure}

From Figure~\ref{fig:experiment}, we can make the following observations. First, the learning curves with $N=20$ are significantly better than those with $N=1$, indicating that running multiple steps of gradient descent in the inner loop (i.e., $N>1$) is crucial for fast convergence. This observation is consistent with our complexity result that $N$-loop is better than No-loop, and $N$-$Q$-loop is better than $Q$-loop, as shown in Table~\ref{tab:results}. The reason is that a more accurate hypergradient estimation can accelerate the convergence rate and lead to a reduction on the Jacobian- and Hessian-vector computational complexity. Second, $N$-$Q$-loop ($N=20$, $Q=20$) and $N$-loop ($N=20$, $Q=1$) achieve a comparable convergence performance, and a similar observation can be made for $Q$-loop  ($N=1$, $Q=20$) and No-loop ($N=1$, $Q=1$). This is also consistent with the complexity result provided in Table~\ref{tab:results}, where different choices of $Q$ do not affect the {\bf dominant} matrix-vector complexity.

 \vspace{0.2cm}
\noindent{\bf Experiments on ITD-BiO.} We consider a hyper-representation problem in \cite{sow2021based}, where the inner problem is to find optimal regression parameters $w$ and the outer procedure is to find the best representation parameters $\lambda$. In specific, the bilevel problem takes the following form: 
\begin{align*}
&\min_{\lambda} \Phi(\lambda) = \frac{1}{2p}\left\|h(X_{V}; \lambda) w^*-Y_{V}\right\|^{2},\;
\text{s.t.}\; w^*=\underset{w}{\operatorname{argmin}} \frac{1}{2q}\|h(X_{T}; \lambda) w-Y_{T}\|^{2}+\frac{\gamma}{2}\|w\|^{2}
\end{align*}
where $X_{T} \in \mathbb{R}^{q \times m}$ and $X_{V} \in \mathbb{R}^{p \times m}$ are synthesized training and validation data, $Y_{T} \in \mathbb{R}^{q}$, $Y_{V} \in \mathbb{R}^{p}$ are their response vectors, and $h(\cdot)$ is a linear transformation. The generation of $X_T,X_V,Y_T,Y_V$ and the experimental setup follow from \cite{sow2021based}. We choose $N=20$ for $N$-$N$-loop ITD and $N=1$ for No-loop ITD. The results are reported with the best-tuned hyperparameters.

\begin{table*}[ht]
    \centering
    \begin{tabular}{cc}
        \begin{minipage}{0.7\textwidth}{
        \begin{tabular}{|c|c|c|c|c|c|}
        \hline
        Algorithm & $k=10$  & $k=50$ & $k=100$ & $k=500$ & $k=1000$ \\
        \hline \hline 
        $N$-$N$-loop ITD & 9.32 & 0.11 & 0.01 & \bf 0.004 & \bf 0.004 \\
        \hline 
        No-loop ITD & 435 & 6.9 & 0.04 & 0.04 & 0.04\\ 
        \hline
        \end{tabular}} 
        \end{minipage} 
    \end{tabular}
    \caption{Validation loss v.s.~the number of iterations for ITD-based algorithms. }
    \label{tb:itdNN}
    \vspace{-1mm}
\end{table*}

\Cref{tb:itdNN} indicates that $N$-$N$-loop with $N=20$ can achieve a small loss value of $0.004$ after $500$ total iterations, whereas No-loop with $N=1$ converges to a much larger loss value of $0.04$. This is in consistence with our theoretical results in \Cref{tab:results_itd}, where  $N=1$ can cause a non-vanishing error.

\vspace{-0.1cm}
\section{Conclusion} 
\vspace{-0.15cm}
In this paper, we study two popular bilevel optimizers AID-BiO and ITD-BiO, whose implementations potentially involve additional loops of iterations within their base-loop update. By developing unified convergence analysis for all choices of the loop parameters, we are able to provide formal comparison among different implementations. Our result suggests that $N$-loops are beneficial for better computational efficiency for AID-BiO and for better convergence accuracy for ITD-BiO. This is in contrast to conventional minimax optimization, where No-loop (i.e., single-base-loop) scheme achieves better computational efficiency. Our analysis techniques can be useful to study other bilevel optimizers such as stochastic optimizers and variance reduced optimizers. 

\bibliographystyle{ref_style.bst}
\bibliography{ref}

\newpage
\appendix
\noindent{\Large\bf Supplementary Materials} 
\vspace{-0.2cm}

\section{Expanded Related Work}\label{app:related_work}
\noindent {\bf Gradient-based bilevel optimization.} A number of gradient-based bilevel algorithms have been proposed via AID- and ITD-based hypergradient approximations. For example, AID-based hypergradient computation~\citep{domke2012generic,pedregosa2016hyperparameter,ghadimi2018approximation,grazzi2020iteration,ji2021bilevel,huang2022efficiently} estimates the Hessian-inverse-vector product by solving a linear system with an efficient iterative algorithm. ITD-based hypergradient computation ~\citep{maclaurin2015gradient,franceschi2017forward,franceschi2018bilevel,finn2017model,shaban2019truncated,ji2020convergence} involves a backpropagation over the inner-loop gradient-based optimization path. Convergence rate of AID- and ITD-based bilevel methods has been studied recently. For example, \cite{ghadimi2018approximation,ji2021bilevel} and \cite{ji2021bilevel,ji2020convergence} analyzed the convergence rate and complexity of AID- and ITD-based bilevel algorithms, respectively. \cite{ji2021lower} characterized the lower complexity bounds for a class of gradient-based bilevel algorithms.
As we mentioned before, previous studies on the convergence rate of deterministic AID-BiO~\citep{ghadimi2018approximation,ji2021bilevel} focused only on $N$-$Q$-loop, and the only convergence rate analysis on ITD-BiO~\citep{ji2021bilevel} was for $N$-$N$-loop. Our study here develops unified convergence analysis for all $N$ and $Q$ regimes.

Some works~\citep{liu2020generic,liu2021value,li2020improved,sow2022constrained} studied the convex inner-level objective function with multiple minimizers. \cite{liu2021towards} proposed an initialization auxiliary method for the setting where the inner-level problem is generally nonconvex.   

\vspace{0.1cm}
\noindent {\bf Stochastic bilevel optimization.} A variety of stochastic bilevel optimization algorithms have been proposed recently. For example, \cite{ghadimi2018approximation,hong2020two,ji2021bilevel} proposed stochastic gradient descent (SGD) type of bilevel algorithms, and analyzed their convergence rate and complexity. Some works \citep{guo2021randomized,guo2021stochastic,yang2021provably,khanduri2021near,chen2021single} then further improved the complexity of SGD type methods using techniques such as variance reduction, momentum acceleration and adaptive learning rate. \cite{sow2021based} proposed a Hessian-free stochastic Evolution Strategies (ES)-based bilevel algorithm with performance guarantee. Although our study mainly focuses on determinstic bilevel optimization, our techniques can be extended to provide refined analysis for stochastic bilevel optimization to capture the order scaling with $\kappa$, which is not captured in most of the above studies on stochastic bilevel optimization.

\vspace{0.1cm}
\noindent {\bf Bilevel optimization for machine learning.} Bilevel optimization has shown promise in many machine learning applications such as hyperparameter optimization~\citep{pedregosa2016hyperparameter,franceschi2018bilevel,ji2021bilevel} and few-shot meta-learning~\citep{finn2017model,snell2017prototypical,rajeswaran2019meta,franceschi2018bilevel,bertinetto2018meta,ji2020convergence,ji2020multi,ji2020convergence}. For example, \cite{snell2017prototypical,bertinetto2018meta} introduced an outer-level procedure to learn a common embedding model for all tasks. \cite{ji2020convergence} analyzed the convergence rate for meta-learning with task-specific adaptation on partial parameters.

\section{Further Specifications on Hyperparameter Optimization Experiments}\label{se:exp_setup}
We follow the setting of~\cite{yang2021provably} to setup the experiment. We first randomly sample $20000$ training samples and $10000$ test samples from MNIST dataset~\citep{lecun1998gradient} with $10$ classes, and then add a label noise on $10\%$ of the data. The label noise is uniform across all labels from label $0$ to  label $9$. We test algorithms with different values of $Q$ and $N$ to verify our theoretical results. 
Every algorithm's  learning rates for inner and outer loops are tuned from the range of $\{0.1,0.01,0.001\}$ and we report the result with the best-tuned learning rates. We run $5$ random seeds and report the average result. All experiments are run over a single NVIDIA Tesla 
P100 GPU. The implementations of our experiments are based on the code of \cite{ji2021bilevel}, which is under MIT License.

\section{Proof Sketch of \Cref{th:mainconverge1}}
The proof of \Cref{th:mainconverge1} contains three major steps, which include 1) decomposing the hypergradient approximation error into the $N$-loop error in estimating the inner-level solution and the $Q$-loop error in solving the linear system approximately, 2) upper-bounding such two types of errors based on the hypergradient approximation errors at previous iterations, and 3) combining all results in the previous steps and proving the convergence guarantee. More detailed steps can be found as below.  

\vspace{0.2cm}
\noindent{\bf Step 1: decomposing hypergradient approximation error.}
\vspace{0.2cm}

\noindent We first show that the hypergradient approximation error at the $k^{th}$ iteration is bounded by 
\begin{align}\label{eq:ggsmidapopocsas}
\|\widehat \nabla \Phi(x_k)- \nabla \Phi(x_k)\|^2 \leq  \underbrace{\Big(3L^2 +\frac{3\rho^2 M^2}{\mu^2} \Big)\|y^*_k-y_k^N\|^2}_{\text{$N$-loop estimation error}} + \underbrace{3L^2\|v_k^*-v_k^Q\|^2.}_{\text{$Q$-loop estimation error}}
\end{align}
where the right hand side contains two types of errors induced by solving the inner-level problem and outer-level linear system. Note that for general choices of $N$ and $Q$, such two errors cannot be guaranteed to be sufficiently small, but fortunately we show via the following results that such errors contain iteratively decreasing components which facilitate the final convergence. 

\vspace{0.2cm}
\noindent{\bf Step 2: upper-bounding linear system approximation error. }
\vspace{0.2cm}

\noindent We then show that the $Q$-loop error $\|v_k^*-v_k^Q\|^2$ for solving the linear system is bounded by
\begin{align}\label{eq:makemefappysc}
\|v_k^Q-v_k^*\|^2\leq  \mathcal{O}\Big( &\big( (1+\eta\mu)(1-\eta\mu)^{2Q}+w\beta^2 \big)\|v_{k-1}^Q-v_{k-1}^*\|^2 \nonumber
\\&+ (\eta^2(1-\alpha\mu)^N+w\beta^2)\|y_{k-1}^*-y_{k-1}^N\|^2 + w\beta^2\|\nabla \Phi(x_{k-1})\|^2 \Big).
\end{align}
Note that if the stepsize $\beta$ is chosen to be sufficiently small, the right hand side of \cref{eq:makemefappysc} contains an {\bf iteratively decreasing} term {\small$(1+\eta\mu)(1-\eta\mu)^{2Q}+w\beta^2 \big)\|v_{k-1}^Q-v_{k-1}^*\|^2$}, an error term {\small $(\eta^2(1-\alpha\mu)^N+w\beta^2)\|y_{k-1}^*-y_{k-1}^N\|^2$} induced by the $N$-loop updates, and gradient norm term $w\beta^2\|\nabla \Phi(x_{k-1})\|^2$ that captures the increment between two adjacent iterations. Similarly, we upper-bound the $N$-loop updating error {\small$\|y^*_k-y_k^N\|^2$} by 
\begin{align}\label{eq:wodimamaysiiics}
\|y_k^N-y_k^*\|^2 \leq  \mathcal{O} \Big(\big((1&+\lambda)(1-\alpha\mu)^N+(1+\lambda^{-1}) \beta^2\big)\|y_{k-1}^N-y^*_{k-1}\|^2  \nonumber
\\&+ (1+\lambda^{-1}) \beta^2\|v_{k-1}^Q-v_{k-1}^*\|^2+(1+\lambda^{-1}) \beta^2\|\nabla \Phi(x_{k-1})\|^2\Big), 
\end{align}
where $\tau=1+\frac{1}{\lambda}$ is inversely proportional to $\lambda$. Note that we introduce an auxiliary variable $\lambda$ in the first error term at the right hand side of \cref{eq:wodimamaysiiics} to allow for {\bf a general choice} of $N$. To see this, to guarantee that {\small$(1+\lambda)(1-\alpha\mu)^N+(1+\lambda^{-1}) \beta^2<1$},  a larger $N$ allows for a smaller $\lambda$. As a result, the outer-level stepsize $\beta$ can be chosen more aggressively, which hence yields a faster convergence rate but at a cost of $N$ steps of $N$-loop updates. On the other hand, if $N$ is chosen to be small, e.g., $N=1$, $\lambda$ needs to be as small as $\lambda = \Theta(\alpha\mu)$. As a result, $\beta$ needs to be smaller, and hence yields a slower convergence rate but with a more efficient $N$-loop update. 

\vspace{0.2cm}
\noindent{\bf Step 3: combining Steps 1 and 2.}
\vspace{0.2cm}

\noindent Combining \cref{eq:ggsmidapopocsas}, \cref{eq:makemefappysc} and \cref{eq:wodimamaysiiics}, we upper-bound the hypergradient estimation error as 
\begin{align*}
\|\widehat \nabla \Phi(x_k)-\nabla\Phi(x_k)\|^2 \leq &\mathcal{O}\big((1-\tau)^k  + \omega\beta^2\sum_{j=0}^{k-1}(1-\tau)^j\|\nabla\Phi(x_{k-1-j}) \|^2\big),
\end{align*}
which, combined with the $L_\Phi$-smoothness property of $\Phi(\cdot)$ and a proper choice of $\beta$, yields the final convergence result. 

\section{Proof of \Cref{th:mainconverge1}}
We first provide some auxiliary lemmas to characterize the hypergradient approximation errors.\begin{lemma}\label{le:vqstart}
Suppose Assumptions~\ref{assum:geo}, \ref{ass:lip}, \ref{high_lip} and \ref{ass:boundGradient} are satisfied. 
Let $v_k^*=(\nabla_y^2g(x_k,y_k^*))^{-1}\nabla_y f(x_k,y_k^*)$ with $y_k^*=\argmin_{y}g(x_k,y)$. Then, we have 
\begin{align*}
\|v_k^Q-v_k^*\|^2\leq & (1+\eta\mu)(1-\eta\mu)^{2Q}\|v_{k-1}^Q-v_{k-1}^*\|^2 \nonumber
\\& + 2\Big(1+\frac{1}{\eta\mu}\Big)C_Q^2\|y_k^*-y_k^N\|^2 
\\&+ 2(1-\eta\mu)^{2Q}\Big(1+\frac{1}{\eta\mu}\Big)\Big(\frac{L}{\mu} + \frac{M\rho}{\mu^2}\Big)^2\Big(\frac{L}{\mu}+1\Big)^2\|x_k-x_{k-1}\|^2,
\end{align*}
where {\small$C_Q = \frac{Q(1-\eta\mu)^{Q-1}\rho M\eta}{\mu} + \frac{1-(1-\eta\mu)^Q(1+\eta Q\mu)}{\mu^2}\rho M + (1-(1-\eta\mu)^Q) \frac{L}{\mu}$}.
\end{lemma}
\begin{proof}
Let $v_k^q$  be the $q^{th}$ $(q=0,...,Q-1)$ GD iterate via solving the linear system $\nabla_y^2 g(x_k,y_k^N) v = 
\nabla_y f(x_k,y^N_k)$, which can be written in the following iterative way.  
\begin{align}\label{eq:ggsmi}
v_k^{q+1} = (I-\eta\nabla_y^2 g(x_k,y_k^N)) v_{k}^q + \eta\nabla_y f(x_k,y_k^N).
\end{align}
Then, by telescoping \cref{eq:ggsmi} over $q$ from $0$ to $Q$ yields 
\begin{align}\label{eq:vkQgg}
v_k^Q = (I-\eta\nabla_y^2g(x_k,y_k^N))^Qv_k^0 + \eta \sum_{q=0}^{Q-1} (I-\eta_y^2g(x_k,y_k^N))^q\nabla_y f(x_k,y_k^N).
\end{align}
Similarly, based on the definition of $v_k^*$, it can be derived that the following equation holds. 
\begin{align}\label{eq:vkstar}
v_k^* = (I-\eta\nabla_y^2 g(x_k,y_k^*))^Q v_k^*  + \eta \sum_{q=0}^{Q-1} (I-\eta_y^2g(x_k,y_k^*))^q\nabla_y f(x_k,y_k^*).
\end{align}
Combining \cref{eq:ggsmi} and \cref{eq:vkQgg}, we next characterize the difference between the estimate $v_k^Q$ and the underlying truth $v_k^*$. In specific, we have 
\begin{align}\label{eq:breakhaars}
\|v_k^Q  - v_k^*\| \overset{(i)}\leq& \big\|(I-\eta\nabla_y^2g(x_k,y_k^N))^Q -(I-\eta\nabla_y^2g(x_k,y_k^*))^Q\big\| \|v_k^*\| + (1-\eta\mu)^Q\|v_k^0-v_k^*\| \nonumber
\\&+\eta \Big\|\sum_{q=0}^{Q-1} (I-\eta_y^2g(x_k,y_k^N))^q - \sum_{q=0}^{Q-1} (I-\eta_y^2g(x_k,y_k^*))^q\Big\| \|\nabla_y f(x_k,y_k^*)\| \nonumber
\\&+\eta L\Big\|\sum_{q=0}^{Q-1} (I-\eta_y^2g(x_k,y_k^N))^q\Big\| \|y_k^*-y_k^N\| \nonumber
\\\overset{(ii)}\leq& \big\|(I-\eta\nabla_y^2g(x_k,y_k^N))^Q -(I-\eta\nabla_y^2g(x_k,y_k^*))^Q\big\| \frac{M}{\mu} + (1-\eta\mu)^Q\|v_{k-1}^Q-v_k^*\| \nonumber
\\&+\eta M \Big\|\sum_{q=0}^{Q-1} (I-\eta_y^2g(x_k,y_k^N))^q - \sum_{q=0}^{Q-1} (I-\eta_y^2g(x_k,y_k^*))^q\Big\| \nonumber
\\& + (1-(1-\eta\mu)^Q) \frac{L}{\mu} \|y_k^*-y^N_k\|.
\end{align}
where $(i)$ follows from the strong convexity of $g(x,\cdot)$ and (ii) follows from Assumption~\ref{ass:boundGradient}, the warm start initialization $v_k^0=v_{k-1}^Q$ and $\|v_k^*\|\leq\|(\nabla_y^2g(x_k,y_k^*))^{-1}\|\|\nabla_y f(x_k,y_k^*)\|\leq \frac{M}{\mu}$.
We next provide an upper bound on the quantity $\Delta_q:= \| (I-\eta_y^2g(x_k,y_k^N))^q-  (I-\eta_y^2g(x_k,y_k^*))^q\|$ in \cref{eq:breakhaars}. In specific, we have 
\begin{align}\label{eq:deltaqas}
\Delta_q \overset{(i)}\leq & (1-\eta\mu)\Delta_{q-1} + (1-\eta\mu)^{q-1}\eta\|\nabla_y^2 g(x_k,y_k^*) -\nabla_y^2 g(x_k,y_k^N)\| \nonumber
\\\leq & (1-\eta\mu)\Delta_{q-1} + (1-\eta\mu)^{q-1}\eta\rho\|y_k^N-y_k^*\|.
\end{align}
where $(i)$ follows from the strong convexity of $g(x,\cdot)$ and Assumption~\ref{high_lip}. 
Telescoping \cref{eq:deltaqas} yields
\begin{align*}
\Delta_q \leq (1 - \eta\mu)^q \Delta_0 + q(1-\eta\mu)^{q-1}\eta\rho\|y_k^N - y_k^*\| = q(1-\eta\mu)^{q-1}\eta\rho\|y_k^N-y_k^*\|,
\end{align*}
which, in conjunction with \cref{eq:breakhaars}, yields
\begin{align}\label{eq:ggaijingyi}
\|v_k^Q - v_k^*\| \leq& Q(1-\eta\mu)^{Q-1}\eta\rho\frac{M}{\mu}\|y_k^N-y_k^*\|  + (1-\eta\mu)^Q\|v_{k-1}^Q - v_k^*\|\nonumber
\\&+\eta M \sum_{q=0}^{Q-1}q(1-\eta\mu)^{q-1}\eta\rho\|y_k^N-y_k^*\| + (1-(1-\eta\mu)^Q) \frac{L}{\mu} \|y_k^*-y^N_k\|.
\end{align}
Based on the facts that $\sum_{q=0}^{Q-1}q x^{q-1} = \frac{1-x^{Q}-Qx^{Q-1}+Qx^Q}{(1-x)^2}>0$,  
we obtain from \cref{eq:ggaijingyi} that 
\begin{align*}
\|v_k^Q-v_k^*\| \leq  & \frac{Q(1-\eta\mu)^{Q-1}\rho M\eta}{\mu}\|y_k^N-y_k^*\|  + (1-\eta\mu)^Q\|v_{k-1}^Q-v_{k-1}^*\|  \nonumber
\\&+ (1-\eta\mu)^Q\|v_{k-1}^*-v_k^*\|
+\frac{1-(1-\eta\mu)^Q(1+\eta Q\mu)}{\mu^2}\rho M\| y_k^N - y_k^*\|  \nonumber
\\&+ (1-(1-\eta\mu)^Q) \frac{L}{\mu} \|y_k^*-y^N_k\| \nonumber
\end{align*}
which, in conjunction with $\|v_k^{*} - v_{k-1}^*\|\leq \big(\frac{L}{\mu} + \frac{M\rho}{\mu^2} \big)\big( \frac{L}{\mu} + 1  \big)\|x_k-x_{k-1}\|$ and using the Young's inequality that $\|a+b\|^2\leq (1+\eta\mu)\|a\|^2 + (1+\frac{1}{\eta\mu})\|b\|^2$, completes the proof of \Cref{le:vqstart}.
\end{proof}

\begin{lemma}\label{le:yknstart}
Suppose Assumptions~\ref{assum:geo} and \ref{ass:lip} are satisfied. 
\begin{align}
\|y_k^* - y_k^N \|^2 \leq (1-\alpha\mu)^N(1+\lambda)\|y_{k-1}^N-y_{k-1}^*\|^2 +(1-\alpha\mu)^N\Big(1+\frac{1}{\lambda}\Big)\frac{L^2}{\mu^2}\|x_k-x_{k-1}\|^2
\end{align}
where $\lambda$ is a positive constant.
\end{lemma}
\begin{proof}
Note that $y^*_k=\argmin_{y}g(x_k,y)$. Using the strong convexity (i.e.,  Assumption~\ref{assum:geo}) and smoothness (i.e., Assumption~\ref{ass:lip}) of $g(x_k,\cdot)$, we have 
\begin{align}
\|y_k^N-y_k^*\|^2 \leq (1-\alpha\mu)^N\|y_k^0-y^*_k\|^2,
\end{align}
which, in conjunction with the warm start initialization $y_k^0=y_{k-1}^N$ and using the Young's inequality, yields 
\begin{align}
\|y_k^N-y_k^*\|^2 \leq & (1+\lambda)(1-\alpha\mu)^N\|y_{k-1}^N-y^*_{k-1}\|^2 + \Big( 1+ \frac{1}{\lambda} \Big) (1-\alpha\mu)^N\|y_{k-1}^*-y_k^*\|^2 \nonumber
\\\overset{(i)}\leq &  (1+\lambda)(1-\alpha\mu)^N\|y_{k-1}^N-y^*_{k-1}\|^2 + \Big( 1+ \frac{1}{\lambda} \Big) (1-\alpha\mu)^N \frac{L^2}{\mu^2}\|x_{k-1}-x_k\|^2,
\end{align}
where $(i)$ follows from Lemma 2.2 in \cite{ghadimi2018approximation}.
\end{proof}
\begin{lemma}\label{le:phixksk}
Suppose Assumptions~\ref{assum:geo}, \ref{ass:lip}, \ref{high_lip} and \ref{ass:boundGradient} are satisfied. Choose parameters such that  $(1+\lambda)(1-\alpha\mu)^N(1+4r(1+\frac{1}{\eta\mu})L^2)\leq 1-\eta\mu$, where the notation $r=\frac{C_Q^2}{(\frac{\rho M}{\mu}+L)^2}$ with $C_Q$ given in \Cref{le:vqstart}. Then, we have the following inequality. 
\begin{align}
\|\widehat \nabla \Phi(x_k)-\nabla\Phi(x_k)\|^2 \leq &3L^2 (1-\eta\mu + 6wL^2\beta^2)^k\delta_0  \nonumber
\\&+ 6wL^2\beta^2\sum_{j=0}^{k-1}(1-\eta\mu 
+ 6wL^2\beta^2)^j\|\nabla\Phi(x_{k-1-j}) \|^2,
\end{align}
where $\delta_0 := \big(1+ \frac{\rho^2 M^2}{L^2\mu^2}  \big)\|y_0^N-y_0^*\|^2 + \|v_0^Q-v_0^*\|^2$ and the notation $w$ is given by 
{\small
\begin{align}\label{de:wklsc}
w = &\Big(1+\frac{1}{\lambda}\Big) (1-\alpha\mu)^N\Big( 1+\frac{\rho^2 M^2}{L^2\mu^2} \Big) \frac{L^2}{\mu^2} \nonumber
\\&+4\Big(1+\frac{1}{\eta\mu}\Big)\frac{L^4}{\mu^2}\Big( 1+\frac{\rho^2 M^2}{L^2\mu^2} \Big) \Big(\frac{4(1-\eta\mu)^{2Q}}{\mu^2}+ r(1-\alpha\mu)^N\Big(1+\frac{1}{\lambda}\Big)\Big).
\end{align}
}
\end{lemma}
\begin{proof}
Combining \Cref{le:vqstart} and \Cref{le:yknstart}, we have 
\begin{align}
\|v_k^Q-v_k^*\|^2 \leq  &(1+\eta\mu)(1-\eta\mu)^{2Q}\|v_{k-1}^Q-v_{k-1}^*\|^2  \nonumber
\\&+ 2(1-\alpha\mu)^N(1+\lambda)\Big(1+\frac{1}{\eta\mu}\Big)C_Q^2\|y_{k-1}^N-y_k^*\|^2 \nonumber
\\&+ 2(1-\alpha\mu)^N\Big(1+\frac{1}{\lambda}\Big)\Big(1+\frac{1}{\eta\mu}\Big)C_Q^2\frac{L^2}{\mu^2}\|x_{k-1}-x_k\|^2 \nonumber
\\&+ 2(1-\eta\mu)^{2Q}\Big(1+\frac{1}{\eta\mu}\Big)\Big(\frac{L}{\mu} + \frac{M\rho}{\mu^2}\Big)^2\Big(\frac{L}{\mu}+1\Big)^2\|x_k-x_{k-1}\|^2, \nonumber
\end{align}
which, in conjunction with $(\frac{L}{\mu}+1)^2\leq 4\frac{L^2}{\mu^2}$ and the notation $r=\frac{C^2_Q}{(\frac{\rho M}{\mu}+L)^2}$, yields 
\begin{align}\label{eq:youdiandongxi}
\|v_k^Q-&v_k^*\|^2 \leq  (1+\eta\mu)(1-\eta\mu)^{2Q}\|v_{k-1}^Q-v_{k-1}^*\|^2  \nonumber
\\&+ 2\Big(1+\frac{1}{\eta\mu}\Big)\frac{L^2}{\mu^2} \Big( L+ \frac{\rho M}{\mu}\Big)^2\Big( \frac{4(1-\eta\mu)^{2Q}}{\mu^2} + r(1-\alpha\mu)^N\Big( 1+\frac{1}{\lambda}\Big) \Big) \|x_k-x_{k-1}\|^2 \nonumber
\\&+ 2(1+\lambda)(1-\alpha\mu)^N \Big(1+\frac{1}{\eta\mu}\Big)\Big(\frac{\rho M}{\mu}+L\Big)^2r\|y_{k-1}^N-y_{k-1}^* \|^2.
\end{align}
Then, combining \Cref{le:yknstart} and \cref{eq:youdiandongxi}, we have 
\begin{align*}
\Big(1&+\frac{\rho^2M^2}{L^2\mu^2}\Big) \|y_k^N - y_k^*\|^2 + \|v_k^Q-v_k^*\|^2 \nonumber
\\\leq&(1+\lambda)(1-\alpha\mu)^N\Big(1+\frac{\rho^2M^2}{L^2\mu^2}\Big) \|y_{k-1}^N-y^*_{k-1}\|^2 \nonumber
\\& + \Big( 1+\frac{1}{\lambda}\Big)(1-\alpha\mu)^N\Big(1+ \frac{\rho^2M^2}{L^2\mu^2} \Big)\frac{L^2}{\mu^2}\|x_{k-1}-x_k \|^2 \nonumber
\\&+ (1+\eta\mu) (1-\eta\mu)^{2Q}\|v_{k-1}^Q-v_{k-1}^*\|^2 \nonumber
\\&+ 4\Big(1+\frac{1}{\eta\mu}\Big)(1+\lambda)\Big(L^2+\frac{\rho^2M^2}{\mu^2}\Big)(1-\alpha\mu)^Nr\|y_{k-1}^N-y_{k-1}^* \|^2 \nonumber
\\&+4\Big(1+\frac{1}{\eta\mu}\Big)\frac{L^4}{\mu^2} \Big( 1+\frac{\rho^2M^2}{\mu^2L^2} \Big)\Big(\frac{4(1-\eta\mu)^{2Q}}{\mu^2} + r(1-\alpha\mu)^N\Big(1+\frac{1}{\lambda}\Big)\Big)\|x_{k-1}-x_k\|^2
\end{align*}
which, in conjunction with the definition of $w$ in \cref{de:wklsc}, yields
\begin{align}\label{eq:idermiterncasc}
\Big(1+&\frac{\rho^2M^2}{L^2\mu^2}\Big) \|y_k^N - y_k^*\|^2 + \|v_k^Q-v_k^*\|^2\nonumber
\\\leq &(1+\lambda)(1-\alpha\mu)^N\Big(1+\frac{\rho^2M^2}{L^2\mu^2}\Big)\Big(1+4r\Big(1+\frac{1}{\eta\mu}\Big)L^2 \Big)\| y_{k-1}^N-y_{k-1}^*\|^2 \nonumber
\\&+(1+\eta\mu)(1-\eta\mu)^{2Q}\|v_{k-1}^Q - v_{k-1}^*\|^2 + w\|x_{k-1}-x_k\|^2.
\end{align}
For notational convenience, we define $\delta_k := \big(1+ \frac{\rho^2 M^2}{L^2\mu^2}  \big)\|y_k^N-y_k^*\|^2 + \|v_k^Q-v_k^*\|^2$ as the per-iteration error induced by $y_k^N$ and $v_k^Q$. Then, recalling that $(1+\lambda)(1-\alpha\mu)^N(1+4r(1+\frac{1}{\eta\mu})L^2)\leq 1-\eta\mu$, we obtain from \cref{eq:idermiterncasc} that 
\begin{align}\label{delta_kkkscas}
\delta_k\leq & (1-\eta\mu)\delta_{k-1} + 2w\beta^2\|\nabla\Phi(x_{k-1})-\widehat\nabla \Phi(x_{k-1})\|^2 + 2w\beta^2\|\nabla \Phi(x_{k-1})\|^2.
\end{align}
Based on the form of $\widehat \nabla\Phi(x_k)$ and $\nabla\Phi(x_k)$ in \cref{hyper-aid} and \cref{trueG}, we have
\begin{align*}
\|\widehat \nabla \Phi(x_k)- \nabla \Phi(x_k)\|^2 \leq& 3\|\nabla_x f(x_k,y^*_k)-\nabla_x f(x_k,y_k^N)\|^2 +3\|\nabla_x \nabla_y g(x_k,y_k^N)\|^2\|v_k^*-v_k^Q\|^2  \nonumber
\\&+ 3\|\nabla_x \nabla_y g(x_k,y^*_k)-\nabla_x \nabla_y g(x_k,y_k^N) \|^2 \|v_k^*\|^2,
\end{align*}
which, in conjunction with Assumptions~\ref{assum:geo},~\ref{ass:lip}, \ref{high_lip} and \ref{ass:boundGradient},  yields
\begin{align}\label{eq:hondulasisi}
\|\widehat \nabla \Phi(x_k)- \nabla \Phi(x_k)\|^2 \leq  \Big(3L^2 +\frac{3\rho^2 M^2}{\mu^2} \Big)\|y^*_k-y_k^N\|^2 + 3L^2\|v_k^*-v_k^Q\|^2.
\end{align}
Substituting \cref{eq:hondulasisi} into \cref{delta_kkkscas} yields
\begin{align*}
\delta_k\leq & (1-\eta\mu+6wL^2\beta^2)\delta_{k-1} + 2w\beta^2\|\nabla \Phi(x_{k-1})\|^2,
\end{align*}
which, by telescoping and using \cref{eq:hondulasisi}, finishes the proof. 
\end{proof}
\subsection*{Proof of \Cref{th:mainconverge1}}
First, based on Lemma 2 in \cite{ji2021bilevel}, we have $\nabla \Phi(\cdot)$ is $L_\Phi$-Lipschitz, where $L_\Phi= L + \frac{2L^2+\rho M^2}{\mu} + \frac{2\rho L M+L^3}{\mu^2} + \frac{\rho L^2 M}{\mu^3}=\Theta(\kappa^3)$. Then, we have 
\begin{align}\label{eq:wodimamwen}
\Phi(x_{k+1}) \leq & \Phi(x_k) + \langle \nabla\Phi(x_k), x_{k+1}-x_k\rangle + \frac{L_\Phi}{2}\|x_{k+1}-x_k\|^2 \nonumber
\\ \leq &   \Phi(x_k)  - \Big(\frac{\beta}{2} - \beta^2L_\Phi\Big) \|\nabla\Phi(x_k) \|^2 +  \Big(\frac{\beta}{2} + \beta^2L_\Phi\Big) \|\nabla\Phi(x_k) - \widehat\nabla\Phi(x_k) \|^2 \nonumber
\\\overset{(i)}\leq &  \Phi(x_k)  - \Big(\frac{\beta}{2} - \beta^2L_\Phi\Big) \|\nabla\Phi(x_k) \|^2 +\Big(\frac{\beta}{2} + \beta^2L_\Phi\Big) 3L^2\delta_0(1-\eta\mu+6wL^2\beta^2)^k \nonumber
\\&+6wL^2\beta^2\Big(\frac{\beta}{2} + \beta^2L_\Phi\Big) \sum_{j=0}^{k-1}(1-\eta\mu+6wL^2\beta^2)^j\|\nabla \Phi(x_{k-1-j})\|^2,
\end{align}
where $(i)$ follows from \Cref{le:phixksk}, $\delta_0$ is defined in  \Cref{le:phixksk} and $w$ is given by \cref{de:wklsc}. Then, telescoping \cref{eq:wodimamwen} over $k$ from $0$ to $K-1$, denoting $x^*=\argmin_x\Phi(x)$ and using, we have 
\begin{align}\label{eq:bigfacssc}
\Big(\frac{\beta}{2} -& \beta^2L_\Phi\Big)\sum_{k=0}^{K-1}\| \nabla \Phi(x_k)\|^2  \nonumber
\\\leq & \Phi(x_0) - \Phi(x^*) + \frac{3L^2\delta_0(\frac{\beta}{2}+\beta^2L_\Phi)}{\eta\mu-6wL^2\beta^2} \nonumber
\\&+ 6wL^2\beta^2 \Big(\frac{\beta}{2} + \beta^2L_\Phi\Big)\sum_{k=0}^{K-1}\sum_{j=0}^{k-1}(1-\eta\mu+6wL^2\beta^2)^j\|\nabla\Phi(x_{k-1-j})\|^2 \nonumber
\\\overset{(i)}\leq & \Phi(x_0) - \Phi(x^*) + \frac{3L^2\delta_0(\frac{\beta}{2}+\beta^2L_\Phi)}{\eta\mu-6wL^2\beta^2}+6wL^2\beta^2 \Big(\frac{\beta}{2} + \beta^2L_\Phi\Big)\frac{\sum_{j=0}^{K-1}\|\nabla\Phi(x_j)\|^2}{\eta\mu - 6 wL^2\beta^2}
\end{align}
where $(i)$ follows because {\small $\sum_{k=0}^{K-1}\sum_{j=0}^{k-1}a_jb_{k-1-j}\leq \sum_{k=0}^{K-1}a_k\sum_{j=0}^{K-1}b_j$}. Rearranging \cref{eq:bigfacssc} yields 
\begin{align}\label{eq:objeinterusc}
\Big(\frac{1}{2} -\beta L_\Phi  - &\frac{6wL^2\beta^2(\frac{1}{2} + \beta L_\Phi)}{\eta\mu - 6wL^2\beta^2} \Big) \frac{1}{K} \sum_{k=0}^{K-1}\|\nabla\Phi(x_k)\|^2 \nonumber
\\& \leq \frac{\Phi(x_0)- \Phi(x^*)}{\beta K}  + \frac{3L^2\delta_0(\frac{1}{2} + \beta L_\Phi)}{\eta\mu-6wL^2\beta^2}\frac{1}{K}.
\end{align}
Note that $(1+\lambda)(1-\alpha\mu)^N(1+4r(1+\frac{1}{\eta\mu})L^2)\leq 1-\eta\mu$ and $r>1$, we have 
\begin{align}
3\eta^2(1-\alpha\mu)^N\Big(1+ \frac{1}{\lambda} \Big) \leq \frac{1-\eta\mu}{1+\lambda} \frac{3\eta^2(1+\frac{1}{\lambda})}{1+4r(1+\frac{1}{\eta\mu})L^2} \leq \frac{1-\eta\mu}{\lambda} \frac{\eta^3\mu}{rL^2},
\end{align}
which, combined with the definitions of $w$ and $\widetilde w$ given by \cref{de:wklsc} and \cref{th:mainconverge1}, yields $w \leq  \widetilde w$. Then, since we set $6\widetilde w L^2\beta^2 \leq \frac{\eta\mu}{3}$ in \Cref{th:mainconverge1}, we have $\frac{6wL^2\beta^2}{\eta\mu - 6wL^2\beta^2}<\frac{6\widetilde wL^2\beta^2}{\eta\mu - 6\widetilde wL^2\beta^2}<\frac{1}{2}$, which, combined with \cref{eq:objeinterusc}, yields
\begin{align*}
\Big(\frac{1}{4}-\frac{3}{2}\beta L_\Phi\Big) \frac{1}{K} \sum_{k=0}^{K-1}\|\nabla\Phi(x_k)\|^2 \leq \frac{\Phi(x_0)- \Phi(x^*)}{\beta K}  + \frac{9L^2\delta_0(\frac{1}{2} + \beta L_\Phi)}{2\eta\mu K},
\end{align*}
which, in conjunction with $\beta\leq \frac{1}{12L_\Phi}$, yields 
\begin{align}\label{eq:kocasqwcas}
 \frac{1}{K} \sum_{k=0}^{K-1}\|\nabla\Phi(x_k)\|^2 \leq \frac{8(\Phi(x_0)- \Phi(x^*))}{\beta K}  + \frac{21L^2\delta_0}{\eta\mu K}.
\end{align}
Based on the updates of $y$ and $v$, we have 
\begin{align}\label{eq:finalaccsw}
\|y_0^N-y_0^*\|^2\leq& \|y_0^0-y^*_0\|^2 = \|y^*_0\|^2 \nonumber
\\\|v_0^Q-v_0^*\|\leq&\|v_0^*\| + \|v_0^Q-(\nabla_y^2 g(x_0,y_0^N))^{-1} \nabla_y f(x_0,y^N_0)\| + \|(\nabla_y^2 g(x_0,y_0^N))^{-1} \nabla_y f(x_0,y^N_0)\|  \nonumber
\\\overset{(i)}\leq & \frac{M}{\mu} + \frac{2}{\mu}(L\|y_0^*\|+M),
\end{align}
where $(i)$ follows because the initialization $v_0^0=0$ and $y_0^0=0$. 
Substituting \cref{eq:finalaccsw} into $\delta_0 := \big(1+ \frac{\rho^2 M^2}{L^2\mu^2}  \big)\|y_0^N-y_0^*\|^2 + \|v_0^Q-v_0^*\|^2$ and \cref{eq:kocasqwcas}, we complete the proof. 

\section{Proof of \Cref{co:coresulstscq1}}
In this case, first note that all choices of $\eta,\alpha,\lambda$ and $N$ satisfy the conditions in \Cref{th:mainconverge1}. 
First recall that $r=\frac{C^2_Q}{(\frac{\rho M}{\mu}+L)^2}$, where $$C_Q = \frac{Q(1-\eta\mu)^{Q-1}\rho M\eta}{\mu} + \frac{1-(1-\eta\mu)^Q(1+\eta Q\mu)}{\mu^2}\rho M + (1-(1-\eta\mu)^Q) \frac{L}{\mu},$$
which, combined with $Q=\Theta(1)$ and $\eta =\Theta(1)$, 
yields $C_Q^2 =\Theta(\kappa^{2})$ 
and hence $r=\Theta(1)$. 
Note that $\widetilde w:=\frac{(1-\eta\mu)\eta\mu}{3\lambda r L^2}\big(  1+ \frac{\rho^2M^2}{L^2\mu^2}\big)\frac{L^2}{\mu^2}+ \big( 1+\frac{1}{\eta\mu}\big)\big( L^2 + \frac{\rho^2M^2}{\mu^2}\big)\big(\frac{16(1-\eta\mu)^{2Q}}{\mu^2} + \frac{4(1-\eta\mu)\eta\mu}{3\lambda L^2}\big) \frac{L^2}{\mu^2}$,  which, combined with $\eta=\frac{1}{L}$ and $\lambda=1$, yields $\widetilde w= \Theta(\kappa^3 + \kappa^7) = \Theta(\kappa^7)$. Based on the choice of $\beta$, we have $$
\beta =\min\big\{ \frac{1}{12L_\Phi},\,\sqrt{ \frac{\eta\mu}{18L^2\widetilde w}}\big\} = \Theta(\kappa^{-4}).$$
Then, we have the following convergence result. 
\begin{align*}
\frac{1}{K}\sum_{k=0}^{K-1}\|\nabla\Phi(x_k)\|^2 =\mathcal{O}\big( \frac{\kappa^4}{K} + \frac{\kappa^3}{K}\big).
\end{align*}
Then, to achieve an $\epsilon$-accurate stationary point, we have $K=\mathcal{O}(\kappa^4\epsilon^{-1})$, and hence we have the following complexity results. 
\begin{itemize}
\item Gradient complexity: $\mbox{\normalfont Gc}(\epsilon)=K(N+2)=\mathcal{\widetilde O}(\kappa^5\epsilon^{-1}).$
\item Matrix-vector product complexities: $$ \mbox{\normalfont MV}(\epsilon)=K + KQ=\mathcal{\widetilde O}\left(\kappa^4\epsilon^{-1}\right).$$
\end{itemize}
Then, the proof is complete. 
\section{Proof of \Cref{co:coresult2}}
Based on the choices of $\alpha,\lambda$ and $\eta\leq\frac{1}{\mu Q}$, recalling $r=\frac{C^2_Q}{(\frac{\rho M}{\mu}+L)^2}$ and using the inequality that $(1-x)^Q\geq 1-Qx$ for any $0<x<1$, we have 
\begin{align*}
r \leq \frac{(\frac{\rho M\eta Q}{\mu}+\eta^2Q^2\rho M +  \eta QL)^2}{(\frac{\rho M}{\mu}+L)^2} \leq 4\eta^2Q^2,
\end{align*}
which, in conjunction with $\eta\leq\frac{1}{128}\frac{\alpha\mu^2}{Q^2L^2}$, yields
\begin{align*}
(1+\lambda)(1-\alpha\mu)^N(1+4r(1+\frac{1}{\eta\mu})L^2)&\leq (1+\lambda)(1-\alpha\mu)^N(1+16(1+\frac{1}{\eta\mu})\eta^2Q^2L^2)
\\&\leq 1-\frac{\alpha\mu}{4} \leq 1- \eta\mu,
\end{align*}
and hence all requirements in \Cref{th:mainconverge1} are satisfied. Also, similarly to the proof of \Cref{co:coresulstscq1}, we have $r=\Theta(1)$, which, combined with $\eta=\Theta(\kappa^{-2})$, yields $\widetilde w=\Theta(\kappa^6 + \kappa^9)= \Theta(\kappa^9)$, and hence 
$$\beta =\min\big\{ \frac{1}{12L_\Phi},\,\sqrt{ \frac{\eta\mu}{18L^2\widetilde w}}\big\}=\Theta(\kappa^{-6}).$$
Then, we have the following convergence result. 
\begin{align*}
\frac{1}{K}\sum_{k=0}^{K-1}\|\nabla\Phi(x_k)\|^2 =\mathcal{O}\big(\frac{\kappa^6}{K} + \frac{\kappa^5}{K}\big).
\end{align*}
Then, to achieve an $\epsilon$-accurate stationary point, we have $K=\mathcal{O}(\kappa^6\epsilon^{-1})$, and hence we have the following complexity results. 
\begin{itemize}
\item Gradient complexity: $\mbox{\normalfont Gc}(\epsilon)=3K=\mathcal{\widetilde O}(\kappa^6\epsilon^{-1}).$
\item Matrix-vector product complexities: $$ \mbox{\normalfont MV}(\epsilon)=K+KQ=\mathcal{\widetilde O}\left(\kappa^6\epsilon^{-1}\right).$$
\end{itemize}
Then, the proof is complete.

\section{Proof of \Cref{th:wodetiantiannass}}
Using an approach similar to \cref{eq:ggaijingyi} in \Cref{le:vqstart}, we have 
\begin{align}\label{eq:choumaomaosccs}
\|v_k^Q-v_k^*\|^2 \leq 2C_Q^2\|y_k^*-y_k^N\|^2 + 2(1-\eta\mu)^{2Q}\|v_k^0-v_k^*\|^2,
\end{align}
where $C_Q$ is defined in \Cref{le:vqstart}. Using the zero initialization $v_k^0$ and based on the fact that $\|v_k^*\|\leq \frac{M}{\mu}$, we obtain from \cref{eq:choumaomaosccs} that 
\begin{align*}
\|v_k^Q-v_k^*\|^2 \leq 2C_Q^2\|y_k^*-y_k^N\|^2 +\frac{ 2(1-\eta\mu)^{2Q}M^2}{\mu^2},
\end{align*}
which, in conjunction with \cref{eq:hondulasisi}, yields
\begin{align}\label{eq:fanzaoleisile}
\|\widehat \nabla \Phi(x_k)- \nabla \Phi(x_k)\|^2 \leq \Big( 3L^2 +\frac{3\rho^2 M^2}{\mu^2} +6L^2C_Q^2 \Big)\|y_k^N-y_k^*\|^2 +  \frac{6L^2(1-\eta\mu)^{2Q}M^2}{\mu^2}.
\end{align}
Then, substituting \cref{eq:fanzaoleisile} into \Cref{le:yknstart}, and using the definition of $\tau$ in \Cref{th:wodetiantiannass}, we have 
\begin{align}\label{eq:haosixnaxiuxi}
\|y_k^* - y_k^N \|^2 \leq & (1-\alpha\mu)^N(1+\lambda)\|y_{k-1}^N-y_{k-1}^*\|^2 +2(1-\alpha\mu)^N\Big(1+\frac{1}{\lambda}\Big)\frac{L^2}{\mu^2}\beta^2\| \nabla \Phi(x_{k-1})\|^2 \nonumber
\\&+ 2(1-\alpha\mu)^N\Big(1+\frac{1}{\lambda}\Big)\frac{L^2}{\mu^2}\beta^2\|\widehat \nabla \Phi(x_{k-1})-\nabla \Phi(x_{k-1})\|^2 \nonumber
\\\leq&\tau \|y_{k-1}^N-y_{k-1}^*\|^2 +2(1-\alpha\mu)^N\Big(1+\frac{1}{\lambda}\Big)\frac{L^2}{\mu^2}\beta^2\| \nabla \Phi(x_{k-1})\|^2  \nonumber
\\&+12(1-\alpha\mu)^N\Big(1+\frac{1}{\lambda}\Big)\frac{L^4M^2}{\mu^4}\beta^2(1-\eta\mu)^{2Q}. 
\end{align}
Telescoping \cref{eq:haosixnaxiuxi} over $k$ yields 
\begin{align*}
\|y_k^* - y_k^N \|^2 \leq &\tau^k \|y_0^* - y_0^N \|^2 + 2(1-\alpha\mu)^N\Big(1+\frac{1}{\lambda}\Big)\frac{L^2}{\mu^2}\beta^2\sum_{j=0}^{k-1}\tau^j\| \nabla \Phi(x_{k-1-j})\|^2   \nonumber
\\ & +\frac{12}{1-\tau}(1-\alpha\mu)^N\Big(1+\frac{1}{\lambda}\Big)\frac{L^4M^2}{\mu^4}\beta^2(1-\eta\mu)^{2Q},
\end{align*}
which, in conjunction with \cref{eq:fanzaoleisile}, $\|y_0^* - y_0^N \|^2\leq (1-\alpha\mu)^N\|y_0-y_0^*\|^2$, the notation of $w$ in \Cref{th:wodetiantiannass} and {\small $\delta_0 = 3 \big( L^2 +\frac{\rho^2 M^2}{\mu^2} +2L^2C_Q^2 \big)(1-\alpha\mu)^N \|y_0^* - y_0 \|^2$}, yields 
\begin{align}\label{eq:jingjijichuruci}
\|\widehat \nabla \Phi(x_k)- \nabla \Phi(x_k)\|^2\leq&\delta_0 \tau^k + 6L^2(1-\eta\mu)^{2Q} \frac{M^2}{\mu^2} +w\beta^2 \sum_{j=0}^{k-1}\tau^j\| \nabla \Phi(x_{k-1-j})\|^2\nonumber
\\ & +\frac{6wL^2M^2}{(1-\tau)\mu^2}(1-\eta\mu)^{2Q}\beta^2.
\end{align}
Then, using an approach similar to \cref{eq:wodimamwen}, we have 
\begin{align}\label{eq:tianbaolaoshiniubi}
\Phi(x_{k+1}) \leq &   \Phi(x_k)  - \Big(\frac{\beta}{2} - \beta^2L_\Phi\Big) \|\nabla\Phi(x_k) \|^2 +  \Big(\frac{\beta}{2} + \beta^2L_\Phi\Big) \|\nabla\Phi(x_k) - \widehat\nabla\Phi(x_k) \|^2 \nonumber
\\\overset{(i)}\leq &  \Phi(x_k)  - \Big(\frac{\beta}{2} - \beta^2L_\Phi\Big) \|\nabla\Phi(x_k) \|^2 +\Big(\frac{\beta}{2} + \beta^2L_\Phi\Big) \delta_0\tau^k \nonumber
\\&+w\beta^2\Big(\frac{\beta}{2} + \beta^2L_\Phi\Big) \sum_{j=0}^{k-1}\tau^j\|\nabla \Phi(x_{k-1-j})\|^2 +  \frac{6L^2M^2}{\mu^2}\Big(\frac{\beta}{2} + \beta^2L_\Phi\Big)(1-\eta\mu)^{2Q}  \nonumber
\\&+\Big(\frac{\beta}{2} + \beta^2L_\Phi\Big)\frac{6wL^2M^2}{(1-\tau)\mu^2}(1-\eta\mu)^{2Q}\beta^2,
\end{align}
where $(i)$ follows from \cref{eq:jingjijichuruci}. Then, 
rearranging the above  \cref{eq:tianbaolaoshiniubi}, we have 
\begin{align*}
\frac{1}{K}\Big(\frac{1}{2}& - \beta L_\Phi\Big)\sum_{k=0}^{K-1} \|\nabla\Phi(x_k) \|^2 \nonumber
\\ \leq & \frac{\Phi(x_0)-\Phi(x^*)}{\beta K} + \frac{1}{K}\Big(\frac{1}{2} + \beta L_\Phi\Big) \frac{\delta_0}{1-\tau} \nonumber
\\&+w\beta^2\Big(\frac{1}{2} + \beta L_\Phi\Big) \frac{1}{K}\sum_{k=0}^{K-1}\sum_{j=0}^{k-1}\tau^j\|\nabla \Phi(x_{k-1-j})\|^2  +   \frac{6L^2M^2}{\mu^2}\Big(\frac{1}{2} + \beta L_\Phi\Big)(1-\eta\mu)^{2Q} \nonumber
\\&+ \Big(\frac{1}{2} + \beta L_\Phi\Big)\frac{6wL^2M^2}{(1-\tau)\mu^2}(1-\eta\mu)^{2Q}\beta^2,
\end{align*} 
which, in conjunction with the inequality that  {\small $\sum_{k=0}^{K-1}\sum_{j=0}^{k-1}a_jb_{k-1-j}\leq \sum_{k=0}^{K-1}a_k\sum_{j=0}^{K-1}b_j$}, yields 
\begin{align}\label{eq:maoxiandaohaowabus}
\Big(\frac{1}{2} -& \beta L_\Phi-w\beta^2\Big(\frac{1}{2} + \beta L_\Phi\Big)\frac{1}{1-\tau}\Big)\frac{1}{K}\sum_{k=0}^{K-1} \|\nabla\Phi(x_k) \|^2 \nonumber
\\ \leq & \frac{\Phi(x_0)-\Phi(x^*)}{\beta K} + \frac{1}{K}\Big(\frac{1}{2} + \beta L_\Phi\Big) \frac{\delta_0}{1-\tau}+ \frac{6L^2M^2}{\mu^2}\Big(\frac{1}{2} + \beta L_\Phi\Big)(1-\eta\mu)^{2Q} \nonumber
\\&+   \Big(\frac{1}{2} + \beta L_\Phi\Big)\frac{6wL^2M^2}{(1-\tau)\mu^2}(1-\eta\mu)^{2Q}\beta^2.
\end{align} 
Using $\beta L_\Phi+w\beta^2\Big(\frac{1}{2} + \beta L_\Phi\Big)\frac{1}{1-\tau} \leq \frac{1}{4}$ in the above \cref{eq:maoxiandaohaowabus} yields
\begin{align*}
\frac{1}{K}\sum_{k=0}^{K-1} \|\nabla\Phi(x_k) \|^2 \leq & \frac{4(\Phi(x_0)-\Phi(x^*))}{\beta K} + \frac{3}{K} \frac{\delta_0}{1-\tau}+ \frac{27L^2M^2}{\mu^2}(1-\eta\mu)^{2Q}, 
\end{align*}
which finishes the proof.

\section{Proof of \Cref{co:xindelargeNlargeQ}}
Note that we choose $N=c_n\kappa\ln \kappa$ and $Q=c_q \kappa\ln \frac{\kappa}{\epsilon}$. Then, for proper constants $c_n$ and $c_q$, we have $\beta L_\Phi<\frac{1}{8}$, $C_Q=\Theta(\kappa^2)$, $\tau=\Theta(1)$ and $w\beta^2\big(\frac{1}{2} + \beta L_\Phi\big)\frac{1}{1-\tau}< \frac{1}{8} $. Then, we have 
\begin{align*}
\frac{1}{K}\sum_{k=0}^{K-1} \|\nabla\Phi(x_k) \|^2 = \mathcal{O}\Big( \frac{\kappa^3}{K} +\epsilon\Big).
\end{align*}
To achieve an $\epsilon$-accurate stationary point, the complexity is given by 
\begin{itemize}
\item Gradient complexity: $\mbox{\normalfont Gc}(\epsilon)=K(N+2)=\mathcal{\widetilde O}(\kappa^4\epsilon^{-1}).$
\item Matrix-vector product complexities: $ \mbox{\normalfont MV}(\epsilon)=K+KQ=\mathcal{\widetilde O}\left(\kappa^{4}\epsilon^{-1}\right).$
\end{itemize}
The proof is then complete. 

\section{Proof of \Cref{co:singQnn1}}
Choose $Q=c_q \kappa\ln \frac{\kappa}{\epsilon}$. Then, for a proper selection of the constant $c_q$, we have $C_Q=\Theta(\kappa^2)$. To guarantee $6\big(1+\frac{1}{\lambda}\big)\frac{L^2}{\mu^2}\big( L^2 +\frac{\rho^2 M^2}{\mu^2} +2L^2C_Q^2 \big)\beta^2 \leq \frac{\alpha\mu}{4}$, we choose $\beta = \Theta(\kappa^{-4})$, which implies $1-\tau =\Theta(\alpha\mu)$. In addition, we have $w=\Theta(\kappa^7)$ and hence $\delta_0/(1-\tau) =\mathcal{O}(\kappa^5)$. Then, we have 
$$\frac{1}{K}\sum_{k=0}^{K-1}\|\nabla\Phi(x_k)\|^2 =\mathcal{O}\Big( \frac{\kappa^5}{K} +\frac{\kappa^4}{K}  +\epsilon\Big).$$
Then, to achieve an $\epsilon$-accurate stationary point, the complexity is given by 
\begin{itemize}
\item Gradient complexity: $\mbox{\normalfont Gc}(\epsilon)=K(N+2)=\mathcal{\widetilde O}(\kappa^5\epsilon^{-1}).$
\item Matrix-vector product complexities: $ \mbox{\normalfont MV}(\epsilon)=K+KQ=\mathcal{\widetilde O}\left(\kappa^{6}\epsilon^{-1}\right).$
\end{itemize}
Then, the proof is complete. 

\section{Proof of \Cref{th:geiwogeofferbossc}}
We first provide two useful lemmas, which are then used to prove \Cref{th:geiwogeofferbossc}.
\begin{lemma}\label{le:ggmisacnaqdacas}
Suppose Assumptions~\ref{assum:geo}, \ref{ass:lip} and \ref{high_lip} are satisfied. Choose inner stepsize $\alpha<\frac{1}{L}$. Then, we have 
\begin{align*}
\Big\|\frac{\partial y_k^N}{\partial x_k}-\frac{\partial y^*(x_k)}{\partial x_k} \Big\| \leq (1-\alpha\mu)^N \Big\| \frac{\partial y^*(x_k)}{\partial x_k} \Big\|  + w_N \|y_k^0-y^*(x_k)\|,
\end{align*}
where we define
\begin{align}\label{eq:notaijkjjjsca}
w_N = \alpha\Big(\rho+\frac{\alpha\rho L(1-(1-\alpha\mu)^{\frac{N}{2}})}{1-\sqrt{1-\alpha\mu}}\Big)  (1-\alpha\mu)^{\frac{N}{2}-1}\frac{1-(1-\alpha\mu)^{\frac{N}{2}}}{1-\sqrt{1-\alpha\mu}}.
\end{align}
\end{lemma}
\begin{proof}
Based on the updates of ITD-based method in~\Cref{alg:main_itd}, we have, for $j=1,....,N$,  
\begin{align*}
\frac{\partial y_k^j}{\partial x_k} =\frac{\partial y_k^{j-1}}{\partial x_k} - \alpha\nabla_x\nabla_y g(x_k,y_k^{j-1}) -\alpha \frac{\partial y_k^{j-1}}{\partial x_k} \nabla_y^2 g(x_k,y_k^{j-1}),
\end{align*}
which, in conjunction with the fact that $\frac{\partial y_k^{0}}{\partial x_k}=0$, yields
\begin{align}\label{eq:rhggkj}
\frac{\partial y_k^N}{\partial x_k} = -\alpha \sum_{j=0}^{N-1}\nabla_x\nabla_y g(x_k,y_k^j) \prod_{i=j+1}^{N-1}(I-\alpha\nabla_y^2g(x_k,y_k^i)).
\end{align}
Then, based on the optimality condition of $y^*(x)$ and using the chain rule, we have 
\begin{align*}
\nabla_x\nabla_y g(x_k,y^*(x_k)) + \frac{\partial y^*(x_k)}{\partial x_k}\nabla_y^2 g(x_k,y^*(x_k))=0,
\end{align*}
which further yields 
\begin{align}\label{eq:rhuuipnasc}
\frac{\partial y^*(x_k)}{\partial x_k} = \frac{\partial y^*(x_k)}{\partial x_k} &\prod_{j=0}^{N-1}(I-\alpha\nabla_y^2 g(x_k,y^*(x_k)))  \nonumber
\\&- \alpha\sum_{j=0}^{N-1} \nabla_x\nabla_y g(x_k,y^*(x_k))\prod_{i=j+1}^{N-1}(I-\alpha\nabla_y^2g(x_k,y^*(x_k))).
\end{align}
For the case where $N=1$, based on \cref{eq:rhggkj} and \cref{eq:rhuuipnasc}, we have
\begin{align}\label{eq:forcase1}
\Big\|\frac{\partial y_k^N}{\partial x_k} -\frac{\partial y^*(x_k)}{\partial x_k} \Big\| \leq (1-\alpha\mu)\Big\|\frac{\partial y^*(x_k)}{\partial x_k} \Big\|  + \alpha \rho \|y_k^0-y^*(x_k)\|.
\end{align}
Next, we prove the case where $N\geq 2$. By subtracting  \cref{eq:rhggkj} by \cref{eq:rhuuipnasc}, we have 
{\footnotesize
\begin{align}\label{eq:zhenhuaqiangssq}
&\Big\|\frac{\partial y_k^N}{\partial x_k}-\frac{\partial y^*(x_k)}{\partial x_k} \Big\| \leq (1-\alpha\mu)^N \Big\| \frac{\partial y^*(x_k)}{\partial x_k} \Big\|  \nonumber
\\&+ \alpha\sum_{j=0}^{N-1}\underbrace{\Big\|\nabla_x\nabla_y g(x_k,y_k^j) \prod_{i=j+1}^{N-1}(I-\alpha\nabla_y^2g(x_k,y_k^i))- \nabla_x\nabla_y g(x_k,y^*(x_k))\prod_{i=j+1}^{N-1}(I-\alpha\nabla_y^2g(x_k,y^*(x_k)))\Big\|}_{\Delta_j},
\end{align}
}
\hspace{-0.12cm}where we define $\Delta_j$ for notational convenience. 
Note that $\Delta_j$ is upper-bounded by  
\begin{align}\label{eq:simpocas}
\Delta_j\leq& (1-\alpha\mu)^{N-1-j}\rho\|y_k^j-y^*(x_k)\| \nonumber
\\&+ L\underbrace{\Big\| \prod_{i=j+1}^{N-1}(I-\alpha\nabla_y^2g(x_k,y_k^i))-\prod_{i=j+1}^{N-1}(I-\alpha\nabla_y^2g(x_k,y^*(x_k)))\Big\|}_{M_{j+1}}.
\end{align}
For notational simplicity, we define a quantity $M_{j+1}$ in \cref{eq:simpocas} for the case where the product index starts from $j+1$. Next we upper-bound $M_{j+1}$ via the following steps. 
\begin{align}\label{eq:gguoprew}
M_{j+1} \leq & (1-\alpha\mu)M_{j+2} + (1-\alpha\mu)^{N-j-2}\alpha\rho\|y_k^{j+1}-y^*(x_k)\| \nonumber
\\\overset{(i)}\leq & (1-\alpha\mu)M_{j+2} + (1-\alpha\mu)^{N-j-2}\alpha\rho (1-\alpha\mu)^{\frac{j+1}{2}}\|y_{k}^0-y^*(x_k)\|\nonumber
\\\leq &  (1-\alpha\mu)M_{j+2} + (1-\alpha\mu)^{N-\frac{j}{2}-\frac{3}{2}}\alpha\rho\|y_k^0-y^*(x_k)\|,
\end{align}
where $(i)$ follows by applying gradient descent to the strongly-convex smooth function $g(x_k,\cdot)$. Telescoping \cref{eq:gguoprew} further yields
\begin{align*}
M_{j+1} \leq &(1-\alpha\mu)^{N-j-2}M_{N-1} + \sum_{i=j+2}^{N-1}(1-\alpha\mu)^{i-j-2}(1-\alpha\mu)^{N-\frac{i-2}{2}-\frac{3}{2}}\alpha\rho\|y_{k}^0-y^*(x_k)\|\nonumber
\\\leq & (1-\alpha\mu)^{N-j-2}M_{N-1} +\sum_{i=0}^{N-j-3}(1-\alpha\mu)^i(1-\alpha\mu)^{N-\frac{j}{2}-\frac{i}{2}-\frac{3}{2}}\alpha\rho\|y_k^0-y^*(x_k)\| \nonumber
\\\leq &  (1-\alpha\mu)^{N-j-2}\alpha\rho (1-\alpha\mu)^{\frac{N-1}{2}}\|y_{k}^0-y^*(x_k)\|  \nonumber
\\&+ \sum_{i=0}^{N-j-3}(1-\alpha\mu)^{N-\frac{j}{2}+\frac{i}{2}-\frac{3}{2}}\alpha\rho\|y_k^0-y^*(x_k)\| \nonumber
\\\leq &  \sum_{i=0}^{N-j-2}(1-\alpha\mu)^{N-\frac{j}{2}+\frac{i}{2}-\frac{3}{2}}\alpha\rho\|y_k^0-y^*(x_k)\|, 
\end{align*}
 which, in conjunction with $\sum_{i=0}^{N-j-2}(1-\alpha\mu)^{\frac{i}{2}}\leq \frac{1-(1-\alpha\mu)^{\frac{N}{2}}}{1-\sqrt{1-\alpha\mu}}$, yields
\begin{align}\label{eq:miplusc1}
M_{j+1} \leq \frac{\alpha\rho(1-(1-\alpha\mu)^{\frac{N}{2}})}{1-\sqrt{1-\alpha\mu}}(1-\alpha\mu)^{N-\frac{j}{2}-\frac{3}{2}}\|y_k^0-y^*(x_k)\|.
\end{align}
Then, substituting \cref{eq:miplusc1} into \cref{eq:simpocas} yields 
\begin{align}\label{eq:dooveritss}
\Delta_j \leq& (1-\alpha\mu)^{N-1-\frac{j}{2}} \rho \|y_k^0-y^*(x_k)\|  \nonumber
\\&+ \frac{\alpha\rho L(1-(1-\alpha\mu)^{\frac{N}{2}})}{1-\sqrt{1-\alpha\mu}} (1-\alpha\mu)^{N-\frac{3}{2}-\frac{j}{2}}\|y_k^0-y^*(x_k)\|.
\end{align}
Summing up \cref{eq:dooveritss} over $j$ from $0$ to $N-1$ yields
\begin{align}\label{eq:youdianyisihas}
\sum_{j=0}^{N-1} \Delta_j \leq &\Big(\rho+\frac{\alpha\rho L(1-(1-\alpha\mu)^{\frac{N}{2}})}{1-\sqrt{1-\alpha\mu}}\Big) \|y_k^0-y^*(x_k)\| (1-\alpha\mu)^{\frac{N}{2}-1}\frac{1-(1-\alpha\mu)^{\frac{N}{2}}}{1-\sqrt{1-\alpha\mu}}. 
\end{align} 
Then, substituting \cref{eq:youdianyisihas} into \cref{eq:zhenhuaqiangssq} and using the notation $w_N$ in~\cref{eq:notaijkjjjsca}, we have
\begin{align}\label{eq:casegeneran}
\Big\|\frac{\partial y_k^N}{\partial x_k}-\frac{\partial y^*(x_k)}{\partial x_k} \Big\| \leq (1-\alpha\mu)^N \Big\| \frac{\partial y^*(x_k)}{\partial x_k} \Big\|  +w_N \|y_k^0-y^*(x_k)\|.
\end{align}
Combining \cref{eq:forcase1} (i.e., $N=1$ case) and \cref{eq:casegeneran} (i.e., $N\geq 2$ case) completes the proof. 
\end{proof}
\begin{lemma}\label{le:yknwcasc}
Suppose Assumptions~\ref{assum:geo}, \ref{ass:lip}, \ref{high_lip} and \ref{ass:boundGradient} hold. Define $$\lambda_N =\frac{4M^2w_N^2+4(1-\frac{1}{4}\alpha\mu)L^2(1+\alpha LN)^2}{1-\frac{1}{4}\alpha\mu-(1-\alpha\mu)^N(1+\frac{1}{2}\alpha\mu)}$$ and $w =\big( 1+\frac{2}{\alpha\mu}\big)\frac{L^2}{\mu^2}(1-\alpha\mu)^N\lambda_N + \frac{4M^2w_N^2L^2}{\mu^2}$, where $w_N$ is given in \cref{eq:notaijkjjjsca}.  Let $\delta_k=\|\widehat \nabla \Phi(x_k) - \nabla \Phi(x_k) \|^2 + \big(\lambda_N - 4L^2\big(1+\alpha L N\big)^2\big)\|y_k^N-y^*(x_k)\|^2$ denote the approximation error at the $k^{th}$ iteration. Choose stepsizes $\beta^2\leq \frac{1-\frac{1}{4}\alpha\mu}{2w}$ and $\alpha\leq\frac{1}{2L}$.  
%
Then, we have 
\begin{align*}
\delta_k \leq \Big(1-\frac{1}{4}\alpha\mu &\Big)^k\delta_0 + J_k (1-\alpha\mu)^{2N} + 2w\beta^2\sum_{j=0}^{k-1} \Big(1-\frac{1}{4}\alpha\mu \Big)^{k-1-j}\|\nabla \Phi(x_j) \|^2,
\end{align*}
where $J_k = \sum_{j=0}^{k-1}  \Big(1-\frac{1}{4}\alpha\mu \Big)^j 4M^2\Big\| \frac{\partial y^*(x_{k-j})}{\partial x_{k-j}} \Big\|^2$ is related to Jacobian matrix of response function.
\end{lemma}
\begin{proof}
First note that using the chain rule, $\widehat \nabla \Phi(x_k)$ and $\nabla\Phi(x_k)$ can be written as 
\begin{align}\label{eq:folmesubtract}
\widehat \nabla \Phi(x_k) =& \nabla_x f(x_k,y_k^N)  + \frac{\partial y_k^N}{\partial x_k} \nabla_y  f(x_k,y_k^N), \nonumber
\\\nabla \Phi(x_k) =& \nabla_x f(x_k,y^*(x_k))  + \frac{\partial y^*(x_k)}{\partial x_k} \nabla_y  f(x_k,y^*(x_k)).
\end{align}
Subtracting two equations in \cref{eq:folmesubtract}, we have 
\begin{align}\label{eq:jbdiaole}
\|\widehat \nabla \Phi(x_k) - \nabla \Phi(x_k) \|\leq  &L\|y_k^N-y^*(x_k)\| \nonumber
\\&+ \Big\|\frac{\partial y^N_k}{\partial x_k} \Big\| L\|y_k^N-y^*(x_k)\| + M\Big\|\frac{\partial y^*(x_k)}{\partial x_k} -  \frac{\partial y_k^N}{\partial x_k}  \Big\|,
\end{align}
which, in conjunction with  {\small  $\big\|\frac{\partial y^N_k}{\partial x_k}\big\|=\|\alpha \sum_{j=0}^{N-1}\nabla_x\nabla_y g(x_k,y_k^j) \prod_{i=j+1}^{N-1}(I-\alpha\nabla_y^2g(x_k,y_k^i))\|\leq \alpha L\sum_{j=0}^{N-1}(1-\alpha\mu)^{N-1-j}\leq\alpha LN$}, yields
\begin{align}\label{eq:haoguiacas}
\|\widehat \nabla \Phi(x_k) - \nabla \Phi(x_k) \|\leq & L\Big(1+\alpha LN\Big) \|y_k^N-y^*(x_k)\|+ M\Big\|\frac{\partial y^*(x_k)}{\partial x_k} -  \frac{\partial y_k^N}{\partial x_k}  \Big\| \nonumber
\\\overset{(i)} \leq & \Big(L+\alpha L^2N\Big) \|y_k^N-y^*(x_k)\| + M\Big\| \frac{\partial y^*(x_k)}{\partial x_k} \Big\| (1-\alpha\mu)^N  \nonumber
\\&+ Mw_N  \|y_k^0-y^*(x_k)\|,
\end{align}
where $(i)$ follows from \Cref{le:ggmisacnaqdacas}. 
Using $\|y_k^0-y^*(x_k)\|=\|y^{N}_{k-1}-y^*(x_{k})\|\leq \|y^{N}_{k-1}-y^*(x_{k-1})\| + \frac{L}{\mu}\|x_k-x_{k-1}\|$ and taking the square on both sides of \cref{eq:haoguiacas}, we have 
\begin{align}\label{eq:reasmulica}
\|\widehat \nabla \Phi(x_k) - \nabla \Phi(x_k) \|^2\leq &4L^2\Big(1+\alpha LN\Big)^2\|y_k^N-y^*(x_k)\|^2 +  4M^2\Big\| \frac{\partial y^*(x_k)}{\partial x_k} \Big\| ^2(1-\alpha\mu)^{2N} \nonumber
\\&+ 4M^2w_N^2\|y^{N}_{k-1}-y^*(x_{k-1})\|^2 + 4M^2w_N^2\frac{L^2}{\mu^2} \|x_{k}-x_{k-1}\|^2.
\end{align}
In the meanwhile, based on \Cref{le:yknstart}, we have,
\begin{align}\label{eq:ggsmicdasacqwc}
\|y_k^N-y^*(x_k)\|^2 \leq &(1-\alpha \mu)^N\Big(1+\frac{1}{2}\alpha\mu\Big)\| y_{k-1}^N-y^*(x_{k-1})\|^2  \nonumber
\\&+ \Big(1+\frac{2}{\alpha\mu}\Big)\frac{L^2}{\mu^2}(1-\alpha\mu)^N\|x_{k-1}-x_k\|^2.
\end{align}
Based on $\alpha\leq \frac{1}{2L}$ and the form of $\lambda_N$ in \Cref{le:yknwcasc}, we have $\lambda_N>4L^2(1+\alpha LN)^2>0$. 
Then, multiplying \cref{eq:ggsmicdasacqwc} by $\lambda_N$  and adding \cref{eq:reasmulica}, we have 
\begin{align}
\|\widehat \nabla & \Phi(x_k) - \nabla \Phi(x_k) \|^2 + \Big(\lambda_N - 4L^2\Big(1+\alpha L N\Big)^2\Big)\|y_k^N-y^*(x_k)\|^2 \nonumber
\\\leq & \Big(1-\frac{1}{4}\alpha\mu\Big) \Big(\lambda_N - 4L^2\Big(1+\alpha L N\Big)^2\Big)\|y_{k-1}^N-y^*(x_{k-1})\|^2 + 4M^2\Big\| \frac{\partial y^*(x_k)}{\partial x_k} \Big\| ^2(1-\alpha\mu)^{2N} \nonumber
\\ &+ \Big( \Big( 1+\frac{2}{\alpha\mu}\Big)\frac{L^2}{\mu^2}(1-\alpha\mu)^N\lambda_N + 4M^2w_N^2\frac{L^2}{\mu^2} \Big) \|x_k-x_{k-1}\|^2, 
\end{align}
which, in conjunction with  $\|x_k-x_{k-1}\|^2=\beta^2\|\widehat \nabla \Phi(x_{k-1})\|^2\leq 2\beta^2\|\widehat \nabla \Phi(x_{k-1})-\nabla \Phi(x_{k-1})\|^2+2\beta^2\|\nabla \Phi(x_{k-1})\|^2$ and using the notation of $w$ in \Cref{le:yknwcasc},  yields
\begin{align}\label{eq:geiwogeonsitebaba}
\|\widehat \nabla &\Phi(x_k) - \nabla \Phi(x_k) \|^2 + \Big(\lambda_N - 4L^2\Big(1+\alpha L N\Big)^2\Big)\|y_k^N-y^*(x_k)\|^2 \nonumber
\\\leq & \Big(1-\frac{1}{4}\alpha\mu\Big) \Big(\lambda_N - 4L^2\Big(1+\alpha L N\Big)^2\Big)\|y_{k-1}^N-y^*(x_{k-1})\|^2 + 4M^2\Big\| \frac{\partial y^*(x_k)}{\partial x_k} \Big\| ^2(1-\alpha\mu)^{2N} \nonumber
\\ &+ 2\beta^2w\|\widehat \nabla \Phi(x_{k-1})-\nabla \Phi(x_{k-1})\|^2+2\beta^2w\|\nabla \Phi(x_{k-1})\|^2. 
\end{align}
Using $\beta^2\leq \frac{1-\frac{1}{4}\alpha\mu}{2w}$ and the notation $\delta_k=\|\widehat \nabla \Phi(x_k) - \nabla \Phi(x_k) \|^2 + \big(\lambda_N - 4L^2\big(1+\alpha L N\big)^2\big)\|y_k^N-y^*(x_k)\|^2$ in the above \cref{eq:geiwogeonsitebaba} yields
\begin{align}\label{eq:tellem28s}
\delta_k \leq 4M^2\Big\| \frac{\partial y^*(x_k)}{\partial x_k} \Big\|^2(1-\alpha\mu)^{2N}  + \Big(1- \frac{1}{4}\alpha\mu\Big)\delta_{k-1} + 2w\beta^2\| \nabla \Phi(x_{k-1})\|^2.
\end{align}
Telescoping the above \cref{eq:tellem28s} over $k$ yields
\begin{align*}
\delta_k \leq &\Big(1-\frac{1}{4}\alpha\mu \Big)^k\delta_0 + \sum_{j=0}^{k-1}  \Big(1-\frac{1}{4}\alpha\mu \Big)^j 4M^2\Big\| \frac{\partial y^*(x_{k-j})}{\partial x_{k-j}} \Big\|^2 (1-\alpha\mu)^{2N}
\\& + 2w\beta^2\sum_{j=0}^{k-1} \Big(1-\frac{1}{4}\alpha\mu \Big)^{k-1-j}\|\nabla \Phi(x_j) \|^2, \nonumber
\end{align*}
which, in conjunction with the definition of $J_k$, finishes the proof. 
\end{proof}  
\subsection*{Proof of \Cref{th:geiwogeofferbossc}}
Choose the same stepsizes $\alpha$ and $\beta$ as in \Cref{le:yknwcasc}. 
Then, based on the smoothness of $\Phi(\cdot)$ (i.e.,  Lemma 2 in \cite{ji2021bilevel}),  we have 
\begin{align}\label{eq:fortelescoascaw}
\Phi(x_{k+1}) \leq &  \Phi(x_k) - \Big(\frac{\beta}{2} - \beta^2L_\Phi\Big)\|\nabla \Phi(x_k)\|^2 + \Big(\frac{\beta}{2}+\beta^2L_\Phi\Big)\|\nabla\Phi(x_k)-\widehat \nabla\Phi(x_k)\|^2 \nonumber
\\\overset{(i)}\leq &\Phi(x_k) - \Big(\frac{\beta}{2} - \beta^2L_\Phi\Big)\|\nabla \Phi(x_k)\|^2 + \Big(\frac{\beta}{2}+\beta^2L_\Phi\Big) \delta_0 \Big(1-\frac{1}{4}\alpha\mu\Big)^k  \nonumber
\\& +
2\Big(\frac{\beta}{2}+\beta^2L_\Phi\Big)w\beta^2 \sum_{j=0}^{k-1}\Big( 1- \frac{1}{4}\alpha\mu\Big)^{k-1-j}\|\nabla\Phi(x_j)\|^2 \nonumber
\\&+ \Big(\frac{\beta}{2}+\beta^2L_\Phi \Big) J_k(1-\alpha\mu)^{2N}
\end{align}
where $(i)$ follows from \Cref{le:yknwcasc} with $\delta_k\geq \|\widehat \nabla \Phi(x_k) - \nabla \Phi(x_k) \|^2$. 
Then, telescoping the above \cref{eq:fortelescoascaw} over $k$ from $0$ to $K-1$ yields
\begin{align}
\Big(\frac{\beta}{2} - \beta^2L_\Phi \Big) \sum_{k=0}^{K-1} &\|\nabla \Phi(x_k)\|^2\leq   \Phi(x_0) - \Phi(x^*) + \frac{4\beta(\frac{1}{2}+\beta L_\Phi)\delta_0}{\alpha\mu} \nonumber
\\&+ \sum_{k=0}^{K-1} J_k \beta\Big(\frac{1}{2}+\beta L_\phi\Big) (1-\alpha\mu)^{2N} \nonumber
\\&+ 2\Big(\frac{\beta}{2} +\beta^2L_\Phi\Big)w\beta^2\sum_{k=0}^{K-1}\sum_{j=0}^{k-1}\Big( 1- \frac{1}{4}\alpha\mu\Big)^{k-1-j}\|\nabla\Phi(x_j)\|^2,
\end{align}
which, combined with $\sum_{k=0}^{K-1}\sum_{j=0}^{k-1}\big( 1- \frac{1}{4}\alpha\mu\big)^{k-1-j}\|\nabla\Phi(x_j)\|^2\leq \frac{4}{\alpha \mu}\sum_{j=0}^{K-1}\|\nabla \Phi(x_j)\|^2$, yields
\begin{align}\label{eq:ggsimidacsaqca1}
\Big(\frac{1}{2} - \beta L_\Phi -& \frac{8}{\alpha\mu}\Big(\frac{1}{2}+ \beta L_\Phi\Big)w\beta^2 \Big) \frac{1}{K}\sum_{k=0}^{K-1}\|\nabla \Phi(x_k)\|^2  \nonumber
\\\leq &\frac{\Phi(x_0)-\Phi(x^*)}{\beta K} + \frac{4(\frac{1}{2}+\beta L_\Phi)\delta_0}{\alpha\mu K} + \big(\frac{1}{2}+\beta L_\Phi\big)(1-\alpha\mu)^{2N}\frac{1}{K}\sum_{k=0}^{K-1}J_k.
\end{align}
Based on the definition of $J_k$ in \Cref{le:yknwcasc}, we have 
\begin{align}\label{eq:dascasfacsaa}
\sum_{k=0}^{K-1}J_k = \sum_{k=0}^{K-1}\sum_{j=0}^{k-1}  \Big(1-\frac{1}{4}\alpha\mu \Big)^j 4M^2\Big\| \frac{\partial y^*(x_{k-j})}{\partial x_{k-j}} \Big\|^2 \overset{(i)}\leq \frac{16M^2}{\alpha\mu}\sum_{k=0}^{K-1}\Big\| \frac{\partial y^*(x_{k})}{\partial x_{k}} \Big\|^2,
\end{align}
where $(i)$ follows from the inequality that  $\sum_{k=0}^{K-1}\sum_{j=0}^{k-1}a_jb_{k-1-j}\leq \sum_{k=0}^{K-1}a_k\sum_{j=0}^{K-1}b_j$. Choose $\beta$ such that $ \beta L_\Phi + \frac{8}{\alpha\mu}\Big(\frac{1}{2}+ \beta L_\Phi\Big)w\beta^2<\frac{1}{4}$. 
In addition, based on \cref{eq:haoguiacas}, recalling the definition that $\delta_0=\|\widehat \nabla \Phi(x_0) - \nabla \Phi(x_0) \|^2 + \big(\lambda_N - 4L^2\big(1+\alpha L N\big)^2\big)\|y_0^N-y^*(x_0)\|^2$, using the fact that $\|\frac{\partial y^*(x_0)}{\partial x_0}\|\leq \frac{L}{\mu}$, we have 
\begin{align}\label{eq:oscaujiajiajiaucsa}
\delta_0 \leq \mathcal{O}\Big( \big(N^2(1-\alpha\mu)^N+w_N^2 +\lambda_N(1-\alpha\mu)^N\big)\|y_0-y^*(x_0)\|^2 +\frac{L^2M^2}{\mu^2}(1-\alpha\mu)^{2N}   \Big).
\end{align}
Recall the definition $\tau =N^2(1-\alpha\mu)^N+w_N^2 +\lambda_N(1-\alpha\mu)^N $. Then, 
substituting \cref{eq:dascasfacsaa} and \cref{eq:oscaujiajiajiaucsa} into \cref{eq:ggsimidacsaqca1} yields
\begin{align}\label{eq:aisstancelemmas}
\frac{1}{K}\sum_{k=0}^{K-1}\|\nabla \Phi(x_k)\|^2 \leq \mathcal{O}\Big(&\frac{\Phi(x_0)-\Phi(x^*)}{\beta K}  + \frac{\tau\|y_0-y^*(x_0)\|^2}{\mu^2K}+ \frac{(1-\alpha\mu)^{2N}}{\mu^3K}\nonumber
\\&+ \frac{M^2}{\alpha\mu}\big(1-\alpha\mu\big)^{2N}\frac{1}{K}\sum_{k=0}^{K-1}\Big\| \frac{\partial y^*(x_{k})}{\partial x_{k}} \Big\|^2\Big), 
\end{align}
which, in conjunction with $\|\frac{\partial y^*(x)}{\partial x}\|\leq \frac{L}{\mu}$, completes the proof. 

\section{Proof of \Cref{co:itdwithlargen}}
Based on the choice of $\alpha$ and $N$ and using $\epsilon<1$, we have $w=\Theta(\sqrt{\epsilon}\kappa^2)$ 
\begin{align}\label{eq:woxchihundnss}
\tau  =  \frac{(\ln\frac{\kappa}{\epsilon})^2}{\kappa^2}\sqrt{\epsilon} + \sqrt{\epsilon} + \frac{\epsilon+\sqrt{\epsilon}\kappa^2(\ln \frac{\kappa}{\epsilon})^2}{\kappa^4} =\mathcal{O}(1),
\end{align}
which, in conjunction with $\beta = \min\Big\{\sqrt{\frac{\alpha\mu}{40 w}},\sqrt{\frac{1-\frac{\alpha\mu}{4}}{2w}},\frac{1}{8L_\Phi} \Big\}$, yields $\beta=\Theta(\kappa^{-3})$. Substituting \cref{eq:woxchihundnss} and $\beta=\Theta(\kappa^{-3})$ into \cref{eq:ggsmidadsadacas} yields
\begin{align*}
\frac{1}{K}\sum_{k=0}^{K-1}\|\nabla \Phi(x_k)\|^2  = \mathcal{O}\Big( \frac{\kappa^3}{K} +\epsilon\Big).
\end{align*}
Then, to achieve an $\epsilon$-accurate stationary point, we have $K=\mathcal{O}(\kappa^3\epsilon^{-1})$, and hence we have the following complexity results. 
\begin{itemize}
\item Gradient complexity: $\mbox{\normalfont Gc}(\epsilon)=K(N+2)=\mathcal{O}(\kappa^4\epsilon^{-1}\ln \frac{\kappa}{\epsilon}).$
\item Matrix-vector product complexities: $$ \mbox{\normalfont MV}(\epsilon)=2KN=\mathcal{O}(\kappa^4\epsilon^{-1}\ln \frac{\kappa}{\epsilon}).$$
\end{itemize}
Then, the proof is complete.

\section{Proof of \Cref{co:itdwithsmalln}}
Based on the choice of $\alpha$ and $N$, we have 
\begin{align*}
w_N &= \Theta(\alpha(\rho+\alpha\rho LN)N ) = \Theta(1),
\\\lambda_N &=\frac{4M^2w_N^2+4(1-\frac{1}{4}\alpha\mu)L^2(1+\alpha LN)^2}{1-\frac{1}{4}\alpha\mu-(1-\alpha\mu)^N(1+\frac{1}{2}\alpha\mu)} =\Theta (\kappa),
\end{align*}
and hence $w=\Theta(\kappa^4)$ and $\tau = \Theta (\kappa)$. Then, we have $\beta=\Theta(\kappa^3)$, and hence we obtain from \cref{eq:ggsmidadsadacas} that 
\begin{align*}
\frac{1}{K}\sum_{k=0}^{K-1}\|\nabla \Phi(x_k)\|^2 = &\mathcal{O}\Big(\frac{\kappa^3}{K}  +  \frac{M^2L^2}{\alpha\mu^3}\Big),
\end{align*}
which finishes the proof.

\section{Proof of \Cref{th:lowerBoundsacasqw}}
We consider the following construction of loss functions.
\begin{align}\label{worst_case_instance}
f(x,y) =& \frac{1}{2}x^TZ_xx + M \mathbf{1}^Ty \nonumber
\\g(x,y)=& \frac{1}{2} y^TZ_yy - L x^Ty + \mathbf{1}^Ty,
\end{align}
where $Z_x = Z_y = \begin{bmatrix}
 L&0  \\
 0  &  \mu  \\ 
\end{bmatrix}$ and $M$ is a positive constant.
First note that the minimizer of inner-level function $g(x,\cdot)$ and the total gradient $\nabla \Phi(x)$ are  given by 
\begin{align}
y^*(x) & = Z_y^{-1}(Lx-\mathbf{1}), \nonumber
\\\nabla\Phi(x) & = Z_xx + L M Z_y^{-1} \mathbf{1}.
\end{align}
Based on the updates of ITD-based method in \Cref{alg:main_itd}, we have, for $t=0,...,N$
\begin{align}\label{eq:caotiancasucq}
y_k^{t} = y_k^{t-1} - \alpha (Z_yy_{k}^{t-1}-Lx_k+\mathbf{1}).
\end{align}
Taking the derivative w.r.t.~$x_k$ on the both sides of \cref{eq:caotiancasucq} yields
\begin{align}\label{eq:mmadacaeq}
\frac{\partial y_k^t}{\partial x_k} = (I-\alpha Z_y)\frac{\partial y_k^{t-1}}{\partial x_k}  + \alpha L I,
\end{align}
Telescoping the above \cref{eq:mmadacaeq} over $t$ from $1$ to $N$ and using the fact that $\frac{\partial y_k^{0}}{\partial x_k}=0$, yields
\begin{align*}
\frac{\partial y_k^N}{\partial x_k} = \alpha L\sum_{t=0}^{N-1}(I-\alpha Z_y)^t,  
\end{align*}
which, in conjunction with the update $x_{k+1} = x_k - \beta \frac{\partial  f(x_k,y_k^N)}{\partial x_k}$, yields
\begin{align}\label{eq:ggmiscaslcadsca}
x_{k+1} = x_k  - \beta \Big(Z_xx_k +\alpha L M\sum_{t=0}^{N-1}(I-\alpha Z_y)^t  \mathbf{1}\Big).
\end{align}
For notational convenience, let $Z_N = \alpha \sum_{t=0}^{N-1}(I-\alpha Z_y)^t$ and $x_0=\mathbf{1}$. Telescoping \cref{eq:ggmiscaslcadsca} over $k$ from $0$ to $K-1$ yields
\begin{align}\label{eq:smidacsadasc}
x_K =& (I-\beta Z_x)^{K}\mathbf{1} - L M\sum_{k=0}^{K-1} (I-\beta Z_x)^{k}\beta Z_N  \mathbf{1} \nonumber
\\ =& (I-\beta Z_x)^{K}\mathbf{1}-LM Z_x^{-1} Z_N  \mathbf{1} + LM\sum_{k=K}^{\infty}(I-\beta Z_x)^{k}\beta Z_N  \mathbf{1}  \nonumber
\\ = &(I-\beta Z_x)^{K}\mathbf{1} - LMZ_x^{-1}Z_N  \mathbf{1} + LM(I-\beta Z_x)^{K}  Z_x^{-1} Z_N  \mathbf{1}.
\end{align}
Rearranging the above \cref{eq:smidacsadasc} yields
\begin{align*}
\|Z_x(x_K + &LMZ_x^{-1}Z_y^{-1}) \mathbf{1}\|^2 
\\=&  \big\|Z_x(I-\beta Z_x)^{K}\mathbf{1} + LM(I-\alpha Z_y)^N Z_y^{-1}  \mathbf{1} + LM(I-\beta Z_x)^{K}  Z_N  \mathbf{1}\big\|^2  \nonumber
\\\geq & L^2M^2\|(I-\alpha Z_y)^N Z_y^{-1}  \mathbf{1}\|^2 +\big\|Z_x(I-\beta Z_x)^{K}\mathbf{1}\big\|^2  + L^2M^2\big\|(I-\beta Z_x)^{K}  Z_N  \mathbf{1}\big\|^2 
\end{align*}
which, in conjunction with $\alpha \leq \frac{1}{L}$, yields
\begin{align}
\|\nabla\Phi(x_K)\|^2 \geq  L^2M^2\|(I-\alpha Z_y)^N Z_y^{-1}  \mathbf{1}\|^2 = \Theta\Big(\frac{L^2M^2}{\mu^2}(1-\alpha\mu)^{2N}\Big), 
\end{align}
which holds for all $K$. 

\end{document}